\documentclass[10pt,twocolumn,letterpaper]{article}

\usepackage[pagenumbers]{cvpr}  %

\usepackage{booktabs}
\usepackage{multirow}
\usepackage{pifont}
\usepackage{graphicx}
\usepackage{tikz}
\usetikzlibrary{positioning,arrows.meta,calc,fit,backgrounds}
\definecolor{blueG}{RGB}{36,113,163}\definecolor{orangeP}{RGB}{230,126,34}\definecolor{greenO}{RGB}{88,184,120}
\usepackage{amsmath,amssymb,amsthm}
\usepackage{algorithm}
\usepackage[noend]{algpseudocode}
\newtheorem{proposition}{Proposition}
\newcommand{\abFCI}{+0.09 [+0.03, +0.18]}
\newcommand{\abFCIHome}{$-$0.00 [$-$0.03, +0.02]}
\newcommand{\abForceCat}{56}

\newcommand{\abForceOurs}{42}

\newcommand{\abQAnsCat}{3}
\newcommand{\abQAnsOurs}{22}
\newcommand{\abQCI}{+0.039 [+0.011, +0.076]}
\newcommand{\abQCIHome}{$-$0.002 [$-$0.012, +0.006]}

\newcommand{\abQFSOurs}{9}

\newcommand{\abTwoPost}{0.270}

\newcommand{\abTwoPt}{0.460}

\newcommand{\boundShare}{39}

\newcommand{\ciPoolHierClean}{[$-$0.09, +0.16]}

\newcommand{\ciPoolPriorClean}{[$-$0.33, $-$0.01]}

\newcommand{\colBox}{40}
\newcommand{\colN}{110}
\newcommand{\colOurs}{23}

\newcommand{\dPoolHierClean}{+0.03}

\newcommand{\dPoolPrior}{$-$0.22}
\newcommand{\dPoolPriorClean}{$-$0.15}

\newcommand{\fbAcc}{48}

\newcommand{\fbHomes}{16}
\newcommand{\fbMajority}{55}
\newcommand{\fbN}{249}
\newcommand{\fbOrient}{+xy}
\newcommand{\fbOursAcc}{86}

\newcommand{\forceHomes}{17}

\newcommand{\gpAccAns}{39}

\newcommand{\gpDrawerN}{106}

\newcommand{\homesBetter}{13}
\newcommand{\homesClear}{11}
\newcommand{\kinFive}{0.99}
\newcommand{\kinOne}{0.68}
\newcommand{\kinTwo}{0.84}
\newcommand{\mapClear}{27}

\newcommand{\mapPriChance}{17}
\newcommand{\mapPriFirst}{77}

\newcommand{\mapRevChance}{29}
\newcommand{\mapRevFirst}{64}

\newcommand{\nHomes}{18}
\newcommand{\nHomesPred}{17}

\newcommand{\obsShareHome}{43}
\newcommand{\obsShareOurs}{41}
\newcommand{\obsSharePrior}{26}

\newcommand{\priorShare}{31}

\newcommand{\ratioCleanBoth}{0.60}

\newcommand{\rigCases}{263}
\newcommand{\rigCrps}{0.544}
\newcommand{\rigFail}{42}
\newcommand{\rigLower}{41}
\newcommand{\rigN}{200}
\newcommand{\rigNoFit}{21}
\newcommand{\rigOurs}{0.323}
\newcommand{\rigOursRmse}{0.485}
\newcommand{\rigRmse}{0.475}
\newcommand{\riskN}{420}
\newcommand{\riskPriorHi}{0.14}
\newcommand{\riskPriorLo}{0.10}

\newcommand{\semCommonCI}{+0.04 [$-$0.00, +0.08]}

\newcommand{\semCommonN}{96}

\newcommand{\spReachShare}{74}

\newcommand{\spShortCI}{+0.048 [$-$0.039, +0.126]}

\newcommand{\tOneAcc}{88.6}
\newcommand{\tOneAccDoor}{93}
\newcommand{\tOneAccDrawer}{84}
\newcommand{\tOneAxDoor}{3.9}
\newcommand{\tOneAxDrawer}{2.5}
\newcommand{\tOneBrier}{0.106}
\newcommand{\tOneBrierPrior}{0.250}
\newcommand{\tOneConfWrong}{0.83}
\newcommand{\tOneHomes}{18}
\newcommand{\tOneN}{306}
\newcommand{\tOnePosDoor}{2.7}

\newcommand{\tTwoHomes}{18}

\newcommand{\tauFactorFit}{1.56}

\newcommand{\timeOursCase}{13}
\newcommand{\timeOursColl}{2.5}
\newcommand{\timeOursDyn}{1.5}
\newcommand{\timeOursJoint}{1.2}
\newcommand{\timeOursQuery}{0.03}
\newcommand{\timeOursResp}{4}
\newcommand{\timeRig}{230}
\newcommand{\timeRigN}{8}

\newcommand{\vRef}{0.14}

\newcommand{\simKitchens}{120}
\newcommand{\simParts}{4428}
\newcommand{\simValid}{3838}
\newcommand{\simBlocked}{503}
\newcommand{\simJammed}{65}
\newcommand{\simInvalid}{22}
\newcommand{\simNoisePos}{4.6}
\newcommand{\simNoiseRot}{2.8}
\newcommand{\simNoiseF}{0.9}
\newcommand{\simNoiseTc}{0.14}
\newcommand{\simTOneN}{4327}

\newcommand{\simOpen}{24}

\newcommand{\simOursOracle}{0.057}

\newcommand{\simForceHomeCI}{+0.020 [+0.008, +0.034]}

\newcommand{\simMixed}{72}
\newcommand{\simForceN}{3707}

\newcommand{\simCollBlocked}{471}
\newcommand{\simFA}{13}
\newcommand{\simMissed}{59}
\newcommand{\simFABox}{15}
\newcommand{\simMissedBox}{61}
\newcommand{\simRigN}{12320}
\newcommand{\simRigParts}{3080}
\newcommand{\simRigKitchens}{120}
\newcommand{\simRigFail}{38}
\newcommand{\simRigAttempt}{3835}
\newcommand{\simNoFit}{717}

\definecolor{cvprblue}{rgb}{0.21,0.49,0.74}
\usepackage[pagebackref,breaklinks,colorlinks,allcolors=cvprblue]{hyperref}

\title{PhysVLMM: Interaction-Conditioned Physical Belief Maps\\for Articulated Objects in Robot Scene Maps}
\author{Dibyendu Ghosh\\}

\begin{document}
\maketitle
\begin{abstract}
Mobile robots must open doors and drawers, yet their scene maps record what each object is and whether it has moved, not how it will respond to a pull. This paper proposes PhysVLMM, a method that maintains a physical belief for every articulated part of a Vision-Language-Motion Map and answers a planner's queries about it, with two algorithms: a map update and a query. Watching a part move provably cannot identify its dynamics, and a measured pull identifies only the pulled part, up to one direction. PhysVLMM therefore (i) keeps that direction as a posterior instead of a point estimate, (ii) passes the evidence of each pull in closed form to the untouched look-alike parts of the map, and (iii) answers whether an action will succeed at a stated risk, or abstains. On real doors and drawers in \nHomes{} homes, it identifies the joint type in \fbOursAcc\,\% of held-out pulls (\fbAcc\,\% for a learned single-view method) and predicts the response to a pull with \rigLower\,\% lower error than differentiable simulation (CRPS \rigOurs{} vs.\ \rigCrps), four orders of magnitude faster (\timeOursCase\,ms vs.\ \timeRig\,s per case). The posterior lowers the final-position error from \abTwoPt{} to \abTwoPost, sharing lowers the force error of untouched parts from \abForceCat\,\% to \abForceOurs\,\%, and at every stated risk from 0.02 to 0.30 the false-safe answers stay within it. Pooling all measured parts of a home does as well as the look-alike groups, and simpler variants predict where a pulled part stops at least as well.
\end{abstract}

\section{Introduction}\label{sec:intro}
Mobile manipulators are expected to fetch objects from cabinets, wardrobes and fridges in buildings they have mapped, and each such task begins by opening an articulated part, such as a door, a drawer, a lid or a window. To plan the opening, a robot has to answer several questions about the part in front of it:
\begin{itemize}
\item Is it movable at all, and has it moved since it was last seen?
\item Which way does it open?
\item Will the pull the robot can apply keep it moving?
\item Will it hit something on the way?
\end{itemize}
The first question is about the object; the others are about how the object responds to an action. A robot that answers them wrongly stalls, damages the furniture, or hits a nearby obstacle.

Over the last decade, several approaches have been developed to model articulated objects. Probabilistic kinematic models select and track the joint type of a part~\cite{sturm2011probabilistic,martinmartin2014online}, and active variants choose actions to reduce their uncertainty~\cite{hausman2015active}. Learned models predict how the points of a part move from a single view~\cite{eisner2022flowbot3d} or segment actionable parts and their poses~\cite{geng2023gapartnet}. At scene scale, open-vocabulary maps tell a robot what and where every object is~\cite{gu2024conceptgraphs}, and Vision-Language-Motion Maps (VLMM) add whether it moves and whether it is movable~\cite{ghosh2026vlmm}. Articulated scene graphs add the joint type and axis of every door and drawer~\cite{buchner2026momasg}, interaction-based methods reconstruct articulated objects as digital twins~\cite{hsu2023ditto}, and physics-aware scene graphs attach materials and mass inferred from images~\cite{li2026physgraph,chopra2026physgs}. However, none of these maps predicts the response of a part to a proposed action. Methods that do estimate dynamics treat one object at a time: ForceTwin~\cite{engelbracht2026forcetwin} identifies inertia, friction, damping and closer forces from a handheld force-sensing gripper, RigPI~\cite{he2026rigpi} refines a vision-language prior with force data by differentiable simulation, and Heiden \etal~\cite{heiden2022probabilistic} keep a posterior over simulator parameters. Each part must therefore be probed before anything is known about its dynamics; ForceTwin explicitly leaves the transfer to unprobed parts to future work~\cite{engelbracht2026forcetwin}.

This limitation is fundamental. Without a measured force, any combination of friction, damping and inertia reproduces any observed motion (\cref{prop:passive}), so watching a part being used identifies its geometry but not its dynamics, and a measured pull informs only the pulled part (\cref{prop:measured}). The only measured evidence about a part that has not been touched therefore comes from other parts of the building that have been measured. This paper proposes PhysVLMM, an interaction-conditioned physical belief map. Each VLMM element is extended with its articulation geometry, a belief over its mechanism and dynamics that retains the alternatives the evidence cannot separate, and a response distribution over motion, contact and collision for a proposed action. Because a building contains many look-alike parts of the same design, installed together, the map shares this evidence among them. The contributions of this paper are:
\begin{itemize}
\item PhysVLMM, a method that maintains a physical belief for every articulated part of a scene map, whose map update (\cref{alg:update}) keeps the direction a one-way pull cannot identify as a posterior and passes each measured pull in closed form to the untouched look-alike parts;
\item its query (\cref{alg:query}), which answers whether an action will keep a part moving and clear of obstacles at a stated risk, or abstains, in \timeOursQuery\,ms per query;
\item an evaluation on real doors and drawers of \nHomes{} homes of the Hoi! dataset~\cite{engelbracht2026hoi} and \simKitchens{} simulated kitchens, against published methods run with their official code and an ablation of each algorithm step. Movability, the first question, is answered by VLMM and is not re-evaluated.
\end{itemize}

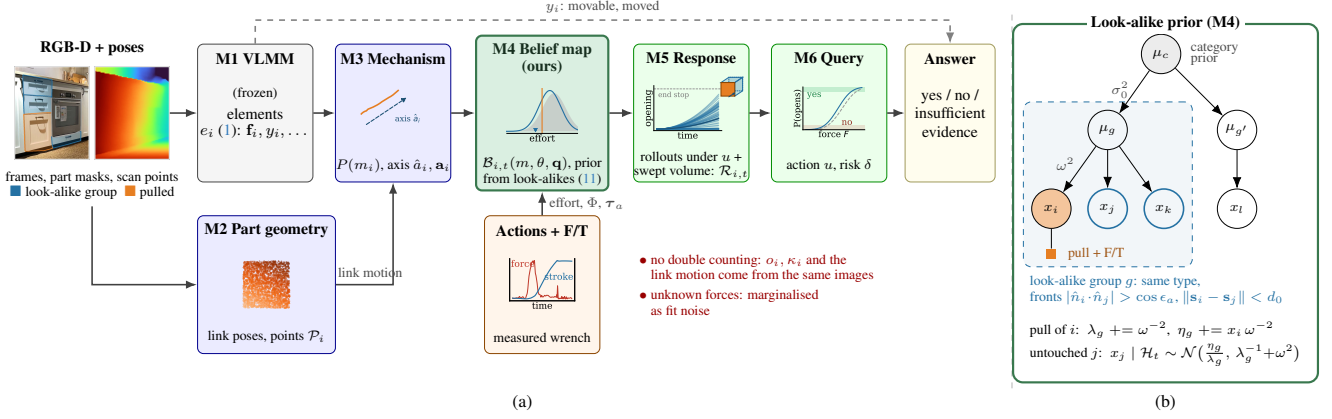
\begin{figure*}[t]\centering
\resizebox{\textwidth}{!}{%
\begin{tikzpicture}[font=\footnotesize, >={Latex[length=2mm]},
  blk/.style={draw=#1!55!black, rounded corners=3pt, fill=#1!8, minimum width=2.05cm, minimum height=2.45cm, inner sep=0pt},
  ttl/.style={font=\footnotesize\bfseries, align=center},
  sub/.style={font=\scriptsize, align=center},
  arr/.style={->, thick, draw=black!75},
  lab/.style={font=\scriptsize, text=black!70}]
\newcommand{\thumb}[5]{%
  \node[ttl] at ($(#1.north)+(0,-0.27)$) {#4};
  \node at ($(#1.center)+(0,0.0)$) {\includegraphics[height=0.95cm, width=1.8cm, keepaspectratio]{#3}};
  \node[sub] at ($(#1.south)+(0,0.3)$) {#5};}
\node[inner sep=0pt] (frame) at (1.25,0) {\includegraphics[width=2.75cm]{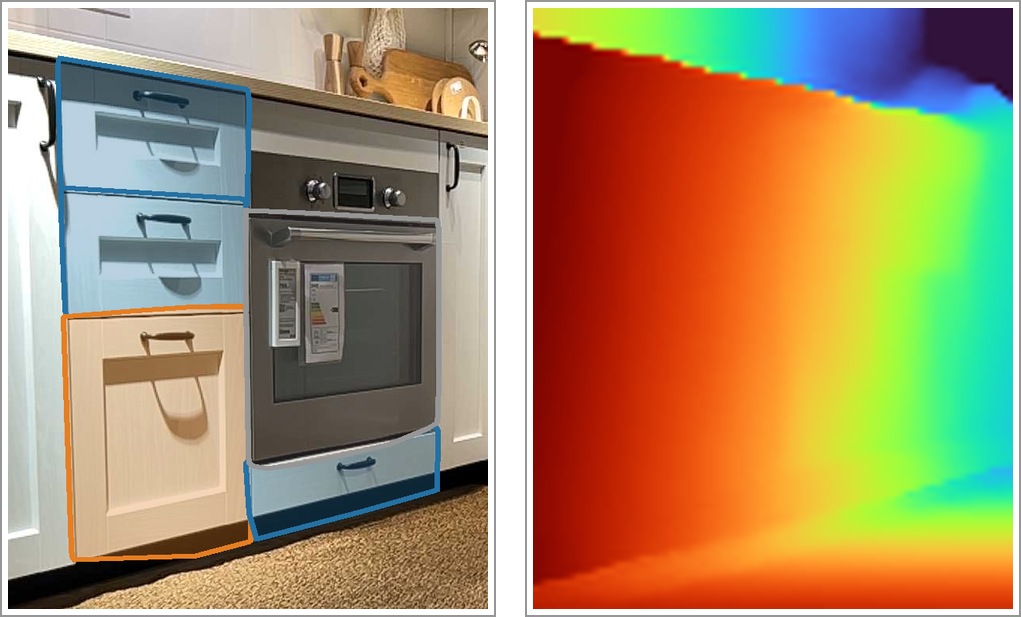}};
\node[ttl, above=2pt of frame] {RGB-D + poses};
\node[sub, below=2pt of frame] (sencap) {frames, part masks, scan points\\ {\color{blueG}$\blacksquare$} look-alike group \ {\color{orangeP}$\blacksquare$} pulled};
\node[blk=gray] (m1) at (4.15,0) {}; \node[ttl] at ($(m1.north)+(0,-0.27)$) {M1 VLMM};
\node[sub] at (m1.center) {(frozen)\\[3pt] \footnotesize elements\\ \footnotesize $e_i$~\eqref{eq:vlmm}: $\mathbf f_i, y_i, \dots$};
\node[blk=blue] (m3) at (6.6,0) {}; \thumb{m3}{}{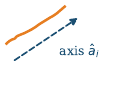}{M3 Mechanism}{$P(m_i)$, axis $\hat a_i, \mathbf a_i$}
\node[blk=greenO, fill=greenO!22, draw=greenO!65!black, line width=1.1pt, minimum width=2.35cm, minimum height=2.75cm] (m4) at (9.25,0) {};
\node[ttl, text=greenO!25!black] at ($(m4.north)+(0,-0.42)$) {M4 Belief map\\[-1pt](ours)};
\node at ($(m4.center)+(0,-0.05)$) {\includegraphics[height=1.0cm]{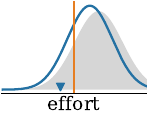}};
\node[sub] at ($(m4.south)+(0,0.36)$) {$\mathcal B_{i,t}(m,\theta,\mathbf q)$, prior\\ from look-alikes~\eqref{eq:pred}};
\node[blk=green] (m5) at (11.9,0) {}; \node[ttl] at ($(m5.north)+(0,-0.27)$) {M5 Response};
\node at ($(m5.center)+(-0.15,0.0)$) {\includegraphics[height=1.0cm]{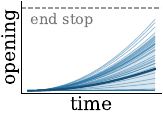}};
\node at ($(m5.center)+(0.72,0.45)$) {\includegraphics[height=0.42cm]{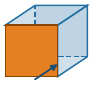}};
\node[sub] at ($(m5.south)+(0,0.3)$) {rollouts under $u$ +\\[-1pt] swept volume: $\mathcal R_{i,t}$};
\node[blk=green] (m6) at (14.35,0) {}; \thumb{m6}{}{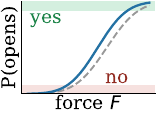}{M6 Query}{action $u$, risk $\delta$}
\node[blk=yellow, minimum width=1.6cm] (ans) at (16.5,0) {}; \node[ttl] at ($(ans.north)+(0,-0.27)$) {Answer};
\node[sub, font=\footnotesize] at (ans.center) {yes / no /\\ insufficient\\ evidence};
\node[blk=blue, minimum width=2.45cm] (m2) at (4.35,-3.05) {}; \thumb{m2}{}{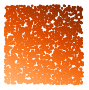}{M2 Part geometry}{link poses, points $\mathcal P_i$}
\node[blk=orange] (act) at (9.25,-3.05) {}; \thumb{act}{}{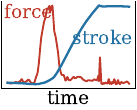}{Actions + F/T}{measured wrench}
\draw[arr] (frame.east) -- (m1.west); \draw[arr] (m1) -- (m3); \draw[arr] (m3) -- (m4); \draw[arr] (m4) -- (m5); \draw[arr] (m5) -- (m6); \draw[arr] (m6) -- (ans);
\draw[arr] (sencap.south) |- (m2.west);
\draw[arr] (m2.east) -| node[lab, pos=0.3, above]{link motion} (m3.south);
\draw[arr] (act.north) -- node[lab, right]{effort, $\Phi$, $\boldsymbol\tau_a$} (m4.south);
\draw[arr, dashed, draw=black!50] (m1.north) -- ++(0,0.42) -| node[lab, pos=0.25, above]{$y_i$: movable, moved} (ans.north);
\node[font=\scriptsize, text=red!60!black, align=left, anchor=west] at (10.85,-3.05) {$\bullet$ no double counting: $o_i, \kappa_i$ and the\\ \phantom{$\bullet$ }link motion come from the same images\\[2pt] $\bullet$ unknown forces: marginalised\\ \phantom{$\bullet$ }as fit noise};
\node[font=\small] (capa) at (8.9,-5.15) {(a)};
\draw[dashed, draw=black!30] (17.6,1.75) -- (17.6,-4.9);
\begin{scope}[shift={(18.3,1.05)},
  var/.style={circle, draw, minimum size=0.7cm, inner sep=0pt, font=\scriptsize}]
\node[var, fill=gray!15] (muc) at (1.95,0) {$\mu_c$};
\node[lab, anchor=west, align=left] at ($(muc.east)+(0.05,0)$) {category\\[-2pt]prior};
\node[var] (g1) at (1.0,-1.35) {$\mu_g$}; \node[var] (g2) at (3.3,-1.35) {$\mu_{g'}$};
\draw[->] (muc) -- node[lab, left, pos=0.45]{$\sigma_0^2$} (g1); \draw[->] (muc) -- (g2);
\node[var, fill=orangeP!45] (x1) at (0.0,-2.75) {$x_i$}; \node[var, draw=blueG, thick] (x2) at (1.0,-2.75) {$x_j$}; \node[var, draw=blueG, thick] (x3) at (2.0,-2.75) {$x_k$};
\node[var] (x4) at (3.3,-2.75) {$x_l$};
\foreach \x in {x1,x2,x3} \draw[->] (g1) -- (\x); \draw[->] (g2) -- (x4);
\node[lab] at (0.25,-1.95) {$\omega^2$};
\node[rectangle, fill=orangeP, minimum size=0.17cm, inner sep=0pt] (f1) at (0.0,-3.55) {}; \draw (x1) -- (f1);
\node[lab, anchor=west, text=orangeP!70!black] at ($(f1.east)+(0.08,0)$) {pull + F/T};
\begin{scope}[on background layer]
\node[draw=blueG, dashed, rounded corners=4pt, fill=blueG!6, fit=(g1)(x1)(x3)(f1), inner sep=4pt] (grp) {};
\end{scope}
\node[lab, align=left, anchor=north west, text=blueG] (gl) at ($(grp.south west)+(0,-0.03)$) {look-alike group $g$: same type,\\ fronts $|\hat n_i\!\cdot\!\hat n_j| > \cos\epsilon_a$, $\|\mathbf s_i - \mathbf s_j\| < d_0$};
\node[align=left, anchor=north west, font=\scriptsize] (eqb) at ($(gl.south west)+(0,-0.05)$) {pull of $i$: \ $\lambda_g \mathrel{+}= \omega^{-2}$, \ $\eta_g \mathrel{+}= x_i\,\omega^{-2}$\\[2pt] untouched $j$: \ $x_j \mid \mathcal H_t \sim \mathcal N\big(\tfrac{\eta_g}{\lambda_g},\, \lambda_g^{-1} {+} \omega^2\big)$};
\node[draw=greenO!65!black, rounded corners=5pt, line width=1.1pt, fit=(muc)(x1)(x4)(eqb)(gl), inner sep=5pt] (boxb) {};
\node[ttl, fill=white, inner sep=2pt] at (boxb.north) {Look-alike prior (M4)};
\end{scope}
\node[font=\small] at (boxb.south |- capa) {(b)};
\end{tikzpicture}}
\caption{Architecture of PhysVLMM on one Hoi! scene: four look-alike drawers (blue), one pulled (orange). (a) Modules M1--M6, from sensors to the answer. (b) Look-alike prior of M4: a pull updates the group effort, and every untouched member receives the predictive effort~\eqref{eq:pred}.}\label{fig:arch}
\end{figure*}

\section{Methodology}\label{sec:method}
\subsection{Problem formulation}\label{sec:problem}
\begin{figure*}[t]\centering
\includegraphics[width=\linewidth]{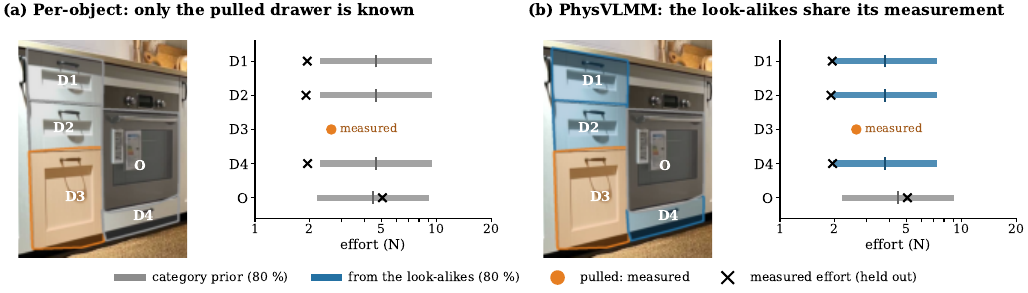}
\caption{Predicted effort after one measured pull (bars: 80\,\% interval; $\times$: measured effort, held out). (a) Per-object: only the pulled drawer D3 is known. (b) PhysVLMM: the pull also informs its look-alikes (blue); the oven O, in another group, keeps the prior.}\label{fig:concept}
\end{figure*}
\paragraph{Scene map.} The robot maintains a Vision-Language-Motion Map~\cite{ghosh2026vlmm} whose elements are
\begin{equation}
e_i = (\mathbf p_i, \mathbf f_i, o_i, \kappa_i, \rho_i, y_i, c_i),
\label{eq:vlmm}
\end{equation}
with position $\mathbf p_i$, semantic feature $\mathbf f_i$, observed-motion score $o_i$ and its reliability $\kappa_i$, movability prior $\rho_i$, and fused motion class $y_i$ with confidence $c_i$. PhysVLMM keeps this seven-tuple unchanged and extends every articulated element to
\begin{align}
e_i^{P} &= \big(e_i,\; \mathcal G_i,\; \mathcal B_{i,t}(m_i, \theta_i, \mathbf q_{i,t}),\; \mathcal R_{i,t}\big), \label{eq:element}\\
\mathcal M_t^{P} &= (e_1^{P}, \dots, e_n^{P}), \nonumber
\end{align}
where $\mathcal G_i$ is the articulation geometry, $\mathcal B_{i,t}$ the physical belief over mechanism $m_i$, parameters $\theta_i$ and joint state $\mathbf q_{i,t} = (q, \dot q)$, and $\mathcal R_{i,t}$ the response attribute defined below.

\paragraph{Part model.} A part is a rigid link $D$ attached to a fixed body by a joint of type $m_i \in \{\text{fixed}, \text{revolute}, \text{prismatic}\}$ with unit axis $\hat a_i$ through the point $\mathbf a_i$. With the joint twist $\xi_i$ and the joint coordinate $q$ (angle or travel, $q = 0$ when closed), the pose of the link and of each of its points $\mathbf x_D$ in the world frame $W$ are
\begin{equation}
T_{WD}(q) = T_{WH}\,\exp(\hat\xi_i\, q)\,T_{HD}(0), \quad \mathbf x_W(q) = T_{WD}(q)\,\mathbf x_D,
\label{eq:poe}
\end{equation}
where $H$ is a frame on the joint axis (\cref{fig:door}). The geometry $\mathcal G_i = (m_i, \hat a_i, \mathbf a_i, \mathcal P_i, [q^-_i, q^+_i])$ also contains the point set $\mathcal P_i$ of the moving link, which is needed for contact, and the joint limits. This model covers any part that rotates about one hinge or slides along one rail, such as doors, drawers, lids, windows and valves; free objects, deformable objects and multi-joint or latched mechanisms are outside it.

The motion of the joint obeys
\begin{equation}
I\ddot q + b\dot q + \tau_c\,\mathrm{sgn}(\dot q) + k q + \tau_0 = J(q)^{\!\top} w + \tau_{\mathrm{ext}} + \tau_{\mathrm{con}},
\label{eq:joint}
\end{equation}
with generalised inertia $I$, viscous damping $b$, Coulomb friction $\tau_c$, and closer stiffness and preload $k$, $\tau_0$; $\theta_i = (I, b, \tau_c, k, \tau_0)$. The applied wrench $w = (F, M)$ at the handle enters as the generalised force $\tau_a = J(q)^{\!\top} w$. For a revolute joint this is $((\mathbf r - \mathbf a_i)\times F + M)\cdot\hat a_i$, with $\mathbf r$ the contact point; for a prismatic joint it is $F\cdot\hat a_i$. In the experiments only the force is used ($M$ is not), with $\mathbf r$ the gripper position. The term $\tau_{\mathrm{ext}}$ collects forces the robot does not measure, such as a person's hand. The contact reactions $\tau_{\mathrm{con}}$ of obstacles and joint stops follow a complementarity condition (Appendix~\ref{sec:s-method}). Without contact, \eqref{eq:joint} is linear in $\theta_i$:
\begin{equation}
\varphi(t)^{\!\top}\theta_i = \tau_a(t) + \tau_{\mathrm{ext}}(t), \quad \varphi = \big(\ddot q, \dot q, \mathrm{sgn}\,\dot q, q, 1\big)^{\!\top}\!.
\label{eq:linear}
\end{equation}

The running force $\bar F_i$ of a part is the handle force needed to keep it moving. In steady opening motion ($\ddot q = 0$, $\dot q = \bar v > 0$) at mid-stroke $\bar q$, with the pull perpendicular to the lever arm and $\tau_{\mathrm{ext}} = 0$, \eqref{eq:joint} gives $\bar F_i = (b\bar v + \tau_c + k\bar q + \tau_0)/r_i$, with $r_i$ the handle radius of a revolute joint and $r_i = 1$ for a prismatic one. Its logarithm $x_i = \log \bar F_i$ is called the effort (measured as in Appendix~\ref{sec:s-proc}).

\paragraph{Belief and response.} The belief is the posterior given all causally available evidence $\mathcal H_t$, \ie observed motion of the part and measured interactions (wrench and motion recorded during a pull):
\begin{equation}
\mathcal B_{i,t}(m, \theta, \mathbf q_t) = p(m, \theta, \mathbf q_t \mid \mathcal H_t).
\label{eq:belief}
\end{equation}
An unmeasured external force is a latent input that is marginalised in the transition. With a zero-mean Gaussian prior that is white in time, the marginal is the noise term of the fit to a measured pull, whose variance is estimated from the residuals (Appendix~\ref{sec:s-method}); predictions of future motion assume no unmeasured force.

For a proposed action $u$ (an applied wrench profile at a contact point over a duration), the response of part $i$ is the joint predictive distribution of the joint motion $\Delta q$ and the contact and collision events $\mathcal C$,
\begin{equation}
\mathcal R_{i,t}(u) = \sum_m \!\int p(\Delta q, \mathcal C \mid u, m, \theta, \mathbf q_t)\, \mathcal B_{i,t}(m, \theta, \mathbf q_t)\, d\theta\, d\mathbf q_t,
\label{eq:response}
\end{equation}
which marginalises the mechanism, parameter and state uncertainty of the belief.

\paragraph{Problem statement.} Given $\mathcal M_t^{P}$ and the evidence $\mathcal H_t$, estimate $\mathcal G_i$, $\mathcal B_{i,t}$ and $\mathcal R_{i,t}$ for every articulated part, including parts that have never been touched, and answer the queries of a pull planner:
\begin{itemize}
\item \emph{Is the part movable?} and \emph{Has it moved?}, from $y_i$;
\item \emph{Which way does it open?}, from $\mathcal G_i$;
\item \emph{Will action $u$ achieve the task safely?}, as the risk test
\begin{equation}
P\big(\text{success}_q \wedge \neg\text{collision} \,\big|\, u, \mathcal H_t\big) \;\ge\; 1 - \delta,
\label{eq:risk}
\end{equation}
\end{itemize}
where $\text{success}_q$ is a task predicate, such as ``the part keeps moving'' or ``the part reaches opening $q^\ast$''. Each answer is ``yes'', ``no'' or ``insufficient evidence''. The collision term is decided by sweeping the link's points $\mathcal P_i$ through the scan (swept volume). With collision ruled out this way, the predicate ``a pull of $F$ newtons keeps the part moving at mid-stroke'' reduces~\eqref{eq:risk} to $P(\bar F_i \le F \mid \mathcal H_t) \ge 1 - \delta$. This is the query evaluated here; it does not cover the higher breakaway force at the start of a pull.

\subsection{Identifiability and design}\label{sec:design}
Two propositions, proved in Appendix~\ref{sec:s-method}, state what the evidence can identify.
\begin{proposition}[Passive observation does not identify dynamics]\label{prop:passive}
Let the generalised force $\tau_a + \tau_{\mathrm{ext}}$ be unmeasured and range over a class that contains all piecewise-continuous functions, let there be no contact, and let $q(t)$ be any twice continuously differentiable trajectory with $\dot q \neq 0$ almost everywhere and finitely many sign changes of $\dot q$ on every bounded interval. Then every $\theta$ is consistent with $q$. If the force has a prior, the likelihood of $\theta$ is this prior evaluated at $\varphi^{\!\top}\theta$ along the trajectory (with a $\theta$-dependent Jacobian), so everything learnt about $\theta$ comes from the assumed distribution of the unmeasured force.
\end{proposition}
A numerical identifiability check is also given there. Observed motion therefore informs $\mathcal G_i$ but not $\theta_i$.
\begin{proposition}[Measured interaction identifies the pulled part only]\label{prop:measured}
Let $\tau_a$ be measured, $\tau_{\mathrm{ext}} = \tau_{\mathrm{con}} = 0$, and $\Phi = [\varphi(t_1), \dots, \varphi(t_N)]^{\!\top}$. Then $\theta_i$ is identifiable if and only if $\Phi$ has full column rank. In particular:
\begin{itemize}
\item for a constant applied force $\tau_a$, $\Phi(\theta_i - \tau_a e_5) = 0$ with $e_5$ the fifth unit vector, so $\Phi$ is rank-deficient (unless $\theta_i = \tau_a e_5$) and $(I, b, \tau_c, k, \tau_0 - \tau_a)$ is identifiable at best up to a common scale;
\item for motion in one direction only, $\tau_c$ cannot be separated from $\tau_0$; only their sum (opening) or their difference (closing) is identified.
\end{itemize}
The measurements of part $i$ carry no information about $\theta_j$, $j \neq i$, unless $\theta_i$ and $\theta_j$ are statistically dependent.
\end{proposition}
A single human pull moves a part in one direction, so by Proposition~\ref{prop:measured} it leaves an unidentifiable direction in $\theta_i$. A point estimate must pick an arbitrary value along it; a belief retains it. The opening effort depends on $\tau_c$ and $\tau_0$ only through their sum, so this direction does not change it: one opening pull fixes the opening effort of the pulled part even though $\theta_i$ stays unknown. Together, the two propositions leave one route to the effort of an untouched part: a statistical dependence on the parts that have been pulled. PhysVLMM builds this dependence from the map by grouping look-alike parts, \ie parts of one type whose fronts are parallel and close together, which are usually of one design and installed together. Type and front come from the part's description and scan points, so a part can be grouped before it has moved. The two propositions fix the three steps of PhysVLMM: keep the direction a pull leaves undecided, share evidence among look-alikes, and answer only at a stated risk. The proposed approach is depicted in \cref{fig:arch}, and \cref{fig:concept} shows how one measured pull reaches the look-alikes of the pulled part.

\begin{algorithm}[t]
\caption{PhysVLMM map update}\label{alg:update}
\begin{algorithmic}[1]
\Require VLMM elements $e_i$ (M1); category parameters $(\mu_c, \sigma_0^2, \omega^2)$ and the classifier from other buildings
\Procedure{ObserveMotion}{$i$, link poses $(\mathbf c_t, R_t)$} \Comment{M2, M3}
  \State fit a straight path and a rigid rotation; BIC of each
  \State $P(m_i{=}\text{rev}) \gets$ logistic classifier on the two fits; axis $\hat a_i, \mathbf a_i$ from the arc fit (revolute) or the path direction (prismatic)
\EndProcedure
\Procedure{Group}{} \Comment{M4}
  \State $(\hat n_i, \mathbf s_i) \gets$ normal and centre of the dominant plane of $\mathcal P_i$; type from the description
  \State $A_{ij} \gets$ same type $\wedge\ |\hat n_i \!\cdot\! \hat n_j| > \cos\epsilon_a \wedge \|\mathbf s_i - \mathbf s_j\| < d_0$
  \State $g(\cdot) \gets$ connected components of $A$ (single linkage); $(\lambda_g, \eta_g) \gets (1/\sigma_0^2, \mu_c/\sigma_0^2)$
\EndProcedure
\Procedure{MeasuredPull}{$i$, force $f(t)$, gripper motion and grip} \Comment{M4}
  \State $x_i \gets$ mean over the pulls of $i$ of $\log$ (median $\|f(t)\|$ while grasped and moving)
  \State $\lambda_{g(i)} \gets \lambda_{g(i)} + 1/\omega^2$; \ $\eta_{g(i)} \gets \eta_{g(i)} + x_i/\omega^2$ \Comment{once per part; a repeat pull replaces $x_i$}
  \State Gaussian prior $(\mathbf m_0, V_0)$ on $\theta_i$: category, shrunk to the measured look-alikes
  \State $(\mathbf m, V) \gets$ posterior~\eqref{eq:thetapost}
\EndProcedure
\end{algorithmic}
\end{algorithm}
\begin{algorithm}[t]
\caption{PhysVLMM query}\label{alg:query}
\begin{algorithmic}[1]
\Require map $\mathcal M_t^{P}$ maintained by \cref{alg:update}; part $j$; action $u$ (handle force $F$); target opening $q^\ast$; risk level $\delta$
\If{query is ``movable?'' or ``moved?''} \Return $y_j$ \Comment{M1}
\EndIf
\If{query is ``which way?''} \Return $\arg\max_m P(m_j)$, $\hat a_j$, $\mathbf a_j$ \Comment{M3}
\EndIf
\If{part $j$ has been pulled} $p_{\text{succ}} \gets \mathbf 1[x_j \le \log F]$ \Comment{measured effort}
\Else\ $p_{\text{succ}} \gets \Phi_{\mathcal N}\big((\log F - \hat x_j)/\sqrt{v_j}\big)$, with $\hat x_j = \eta_{g(j)}/\lambda_{g(j)}$, $v_j = 1/\lambda_{g(j)} + \omega^2$ \Comment{\eqref{eq:pred}}
\EndIf
\If{part $j$ has been pulled} \Comment{M5}
  \State sample $\theta^{(s)} \sim \mathcal N(\mathbf m, V)$; roll out $q^{(s)}(t)$ under $u$ $\to$ motion samples of $\mathcal R_{j,t}$, returned with the answer
\EndIf
\State sweep $\mathcal P_j(q)$ through the scan; $q_{\mathrm{free}} \gets$ first $q$ at which a swept point comes within 1\,cm of the scan \Comment{M5}
\State $p \gets p_{\text{succ}} \cdot \mathbf 1[q_{\mathrm{free}} \ge q^\ast]$ \Comment{\eqref{eq:risk}}
\If{$p \ge 1 - \delta$} \Return yes \ElsIf{$p \le \delta$} \Return no \Else\ \Return insufficient evidence \Comment{M6}
\EndIf
\end{algorithmic}
\end{algorithm}

\subsection{Algorithm}\label{sec:algo}
\paragraph{Kinematic belief (M3).} In \cref{alg:update}, a straight path and a rigid rotation are fitted to the observed link poses; a classifier on the two fits gives $P(m_i = \text{rev})$, and the better fit the axis.

\paragraph{Posterior of a pulled part (M4).} A measured pull gives the regressors $\Phi$ and the generalised forces $\boldsymbol\tau_a$ of~\eqref{eq:linear}. With a Gaussian prior $(\mathbf m_0, V_0)$, the category prior shrunk to the measured look-alikes, the posterior is
\begin{equation}
V = \big(V_0^{-1} + \Phi^{\!\top}\Phi/\sigma^2_{\mathrm{eff}}\big)^{-1}, \quad \mathbf m = V\big(V_0^{-1}\mathbf m_0 + \Phi^{\!\top}\boldsymbol\tau_a/\sigma^2_{\mathrm{eff}}\big),
\label{eq:thetapost}
\end{equation}
with $\sigma^2_{\mathrm{eff}}$ the inflated residual variance. Along the direction left unidentified, $\Phi^{\!\top}\Phi$ is (nearly) singular, so the posterior keeps the prior's spread there. The response~\eqref{eq:response} averages rollouts of~\eqref{eq:joint} over posterior samples.

\paragraph{Evidence sharing (M4).} The efforts of look-alike parts share a group effort,
\begin{equation}
\mu_g \sim \mathcal N(\mu_c, \sigma_0^2), \qquad x_i \mid \mu_{g(i)} \sim \mathcal N(\mu_{g(i)}, \omega^2),
\label{eq:model}
\end{equation}
with category mean $\mu_c$, between-group spread $\sigma_0^2$ and within-group spread $\omega^2$~\cite{gelman2013bayesian}, from other buildings ($\sigma_0^2$, $\omega^2$ by marginal likelihood over a grid). A group's natural parameters $(\lambda_g, \eta_g)$ start at $(1/\sigma_0^2, \mu_c/\sigma_0^2)$, a measured part $i$ adds $(1/\omega^2, x_i/\omega^2)$, and an untouched part $j$ has the predictive effort
\begin{equation}
x_j \mid \mathcal H_t \sim \mathcal N\big(\eta_{g(j)}/\lambda_{g(j)},\; \lambda_{g(j)}^{-1} + \omega^2\big).
\label{eq:pred}
\end{equation}
One pull thus updates its whole group with two additions.

\paragraph{Risk-bounded query (M5, M6).} In \cref{alg:query}, the success probability of a force $F$ on part $j$ comes from $x_j$ if $j$ was pulled and from~\eqref{eq:pred} otherwise; the collision term sweeps $\mathcal P_j(q)$ through the scan. The answer to~\eqref{eq:risk} is ``yes'' if its probability is at least $1 - \delta$, ``no'' if at most $\delta$, and ``insufficient evidence'' otherwise. If the predictive distribution is calibrated, each ``yes'' is false-safe with probability at most $\delta$, and so is the expected false-safe rate of the ``yes'' answers.

\paragraph{Cost.} A pull costs one bounded least-squares fit, one $5 \times 5$ solve and two additions; a query costs one normal CDF and one k-d-tree search per swept point; a response integrates the 1-D joint equation per posterior sample. Nothing is optimised through a simulator, so a case takes milliseconds (\cref{sec:published}); fitting details are in Appendix~\ref{sec:s-method}.

\section{Results and discussions}\label{sec:results}
To evaluate the performance of PhysVLMM, the proposed approach was implemented in Python 3.10 with NumPy and SciPy, and all experiments were run on one laptop (AMD Ryzen 7 7840HS, 8 cores, 32\,GB RAM, NVIDIA RTX 4060 Laptop GPU with 8\,GB); the physical belief layer (M2--M6) runs on the CPU only. Two data sources are used: real doors and drawers of the Hoi! dataset~\cite{engelbracht2026hoi}, and simulated RoboCasa kitchens~\cite{nasiriany2024robocasa} in ManiSkill~3~\cite{tao2025maniskill3}, which add fixed robot pulls and collision ground truth (\cref{sec:setup}). PhysVLMM is compared on five quantities: (i) the joint type and axis, with FlowBot3D~\cite{eisner2022flowbot3d} and GAPartNet~\cite{geng2023gapartnet}; (ii) the response of a part to a measured pull, with RigPI~\cite{he2026rigpi}; (iii) the force needed to move a part that has not been touched and (iv) the answer to ``will this pull keep it moving?'', with ablations, since no published method predicts these for an untouched part; and (v) collision along the opening, with a bounding-box test. Published methods are run with their official code; the adapters only convert the inputs (Appendix~\ref{sec:s-baselines}). The results below follow this order; \cref{tab:published} reports simulation first and real data second.
\subsection{Experimental setup}\label{sec:setup}
\paragraph{Simulated data.} In \simKitchens{} RoboCasa kitchens, a robot pulls every door and drawer from the closed state with a fixed force of 3--12\,N for 1\,s, so most openings stop short (median \simOpen\,\% of the stroke). The physics of each part is drawn from a group model fitted on Hoi!, poses and forces carry the noise measured on Hoi!, and collisions with other fixtures are known (Appendix~\ref{sec:s-sim}).

\paragraph{Real data.} The real experiments use the Hoi! dataset~\cite{engelbracht2026hoi}, in which people opened and closed real doors and drawers in furnished homes with a handheld gripper. The gripper carries a 6-axis force/torque sensor (100\,Hz), an Aria headset whose SLAM gives the gripper trajectory, and a trigger that measures the grip. Laser scans of every home provide each part's joint type and axis, its 3D points and a short description. Of the 20 released homes, \nHomes{} contain force recordings. Two development homes only set the inputs of the compared methods and the collision tolerance; joint type and collision are scored on the other \fbHomes{} homes, force and query on the \forceHomes{} with look-alike groups, final position on all \tTwoHomes. Each recorded window also contains the walk to and from the handle, so only the grasped phase is used (\cref{fig:trace}). No new data were collected from people.

\begin{figure}[t]\centering
\includegraphics[width=\linewidth]{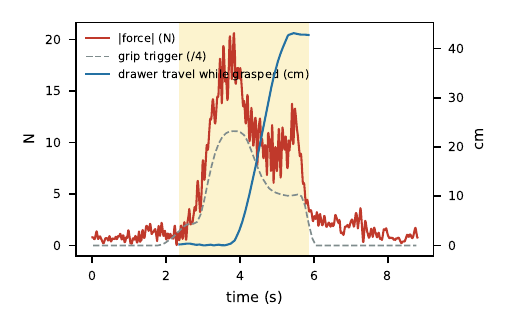}
\caption{One recorded pull (wardrobe drawer); shaded: grasped phase.}\label{fig:trace}
\end{figure}

\paragraph{Protocol.}
\begin{itemize}
\item \emph{Oracle perception.} Part identities, descriptions and 3D points come from the scan annotations, for PhysVLMM and for every compared method, so that physical inference is evaluated separately from perception. Every part in the data moves, so the \emph{fixed} joint type is not evaluated.
\item \emph{No leakage.} The joint annotations are used only for scoring and for one dataset check (scan instances whose index points to an annotated part of another class are excluded for every method). Window selection, look-alike groups and the motion score use only the gripper path, the part descriptions and the scan points. Every learned quantity (calibration, category priors, group spreads) is fitted on homes other than the one evaluated.
\item \emph{Known-answer tests.} Every component was first checked on synthetic cases with known answers (Appendix~\ref{sec:s-tests}); the tests found three errors, all fixed before these results.
\end{itemize}

\paragraph{Metrics.} Joint type: accuracy. Axis: median angle error. Force of an untouched part: typical error $\exp(\overline{|\Delta \log F|}) - 1$. Query ``will $F$\,N keep it moving?'' ($F \in \{3,5,8\}$\,N, $\delta = 0.1$): Brier score; the effort and collision tests of~\eqref{eq:risk} are scored separately; the combined test and the pulled-part branch of \cref{alg:query} are not evaluated, and the collision term is a deterministic check. Final position: CRPS of the final joint position of a held-out pull, as a fraction of the stroke (whole-trajectory error in \cref{tab:t2}). Differences are paired, with 95\,\% confidence intervals from bootstrapping homes.

\subsection{Comparison with the state of the art}\label{sec:published}
\begin{table}[t]\centering\small\setlength{\tabcolsep}{3pt}
\caption{Comparison with published methods (official code), same cases per block. Joint type: \% correct; axis, hinge: median error; response: CRPS of the final position (RigPI minus PhysVLMM, 95\,\% CI: +0.249 [+0.239, +0.261] in simulation, +0.221 [+0.126, +0.312] on real data); time: median per case on one thread of the same CPU, measured on real data. --: not predicted. $^{a}$Answers 42\,\% / 45\,\% of the parts / pulls. $^{b}$On its 106 drawer pulls (PhysVLMM: 2.8$^\circ$). $^{c}$GPU only. $^{d}$Joint, physical belief and response. $^{*}$Code not released; $^{\ddagger}$released code lacks the articulation stage; $^{\dagger}$no joint response.}\label{tab:published}
\resizebox{\linewidth}{!}{\begin{tabular}{lccccc}\toprule
 & joint type & axis, drawer / & hinge & response & time \\
method & $\uparrow$ & door ($^\circ$) $\downarrow$ & (cm) $\downarrow$ & $\downarrow$ & $\downarrow$ \\\midrule
\multicolumn{6}{l}{\emph{Simulation (120 RoboCasa kitchens, 4327 parts, 12320 pulls)}} \\
FlowBot3D~\cite{eisner2022flowbot3d} & 66\,\% & 15.4 / 33.8 & 14.7 & -- & \\
GAPartNet~\cite{geng2023gapartnet} & 35\,\%$^{a}$ & 70.6 / -- & -- & -- & \\
RigPI~\cite{he2026rigpi} & -- & -- & -- & 0.310 & \\
\textbf{PhysVLMM} & \textbf{99\,\%} & \textbf{0.5 / 0.6} & \textbf{0.2} & \textbf{0.062} & \\\midrule
\multicolumn{6}{l}{\emph{Real data (Hoi!, 249 pulls; response: 200 pulls)}} \\
FlowBot3D~\cite{eisner2022flowbot3d} & 48\,\% & 25 / 35 & 11.6 & -- & 42\,ms \\
GAPartNet~\cite{geng2023gapartnet} & 17\,\%$^{a}$ & 16$^{b}$ / -- & -- & -- & 54\,ms$^{c}$ \\
RigPI~\cite{he2026rigpi} & -- & -- & -- & 0.544 & 230\,s \\
\textbf{PhysVLMM} & \textbf{86\,\%} & \textbf{2.7 / 4.0} & \textbf{2.7} & \textbf{0.323} & \textbf{13\,ms}$^{d}$ \\\midrule
ForceTwin~\cite{engelbracht2026forcetwin}$^{*}$, IAAO~\cite{zhang2025iaao}$^{\ddagger}$ & \multicolumn{5}{c}{not run} \\
PhysGS~\cite{chopra2026physgs}, NGFF~\cite{li2026ngff} & \multicolumn{5}{c}{not comparable$^{\dagger}$} \\
\bottomrule\end{tabular}}\end{table}

\begin{figure*}[t]\centering
\includegraphics[width=\linewidth]{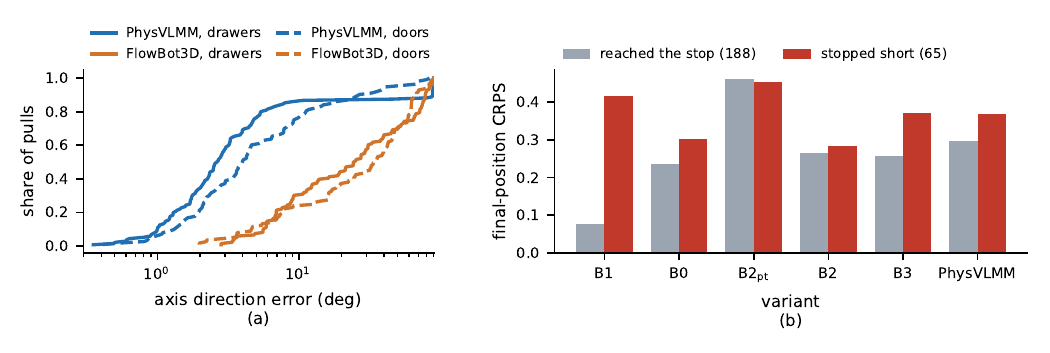}
\caption{(a) Axis error on the held-out homes, PhysVLMM vs.\ FlowBot3D. (b) Final-position error, split by whether the pull reached the stop.}\label{fig:results}
\end{figure*}
\Cref{tab:published} compares PhysVLMM with the published methods on the same cases.
\begin{itemize}
\item \emph{Joint type and axis.} FlowBot3D and GAPartNet infer the joint from one static scan and were trained on simulated objects, whereas PhysVLMM observes the link moving. FlowBot3D is less accurate than always answering ``drawer'' (\fbMajority\,\%), and \cref{fig:results}a shows the distribution of its axis errors over the held-out pulls; the input orientation it expects was chosen on the development homes. GAPartNet labels most real parts as no actionable part and is right on \gpAccAns\,\% of the pulls it answers; its drawer error comes from \gpDrawerN{} drawer pulls, and its released code gives no rule for a door's hinge axis.
\item \emph{Response.} RigPI fits the mass, inertia and Coulomb friction of the link to the training pull by differentiable simulation. Its model has no damping, spring or preload, which door closers and drawer runners have. Over the whole trajectory, its error is similar to that of PhysVLMM (\rigRmse{} vs.\ \rigOursRmse); the gap is in where the part ends. Its vision-language prior is not released, so the prior was computed from the size of the scanned panel. Cases where RigPI returned no prediction (\rigFail{} of \rigCases, a non-finite gradient) or the training pull is too short for the fit of PhysVLMM (\rigNoFit) are excluded for all methods.
\item \emph{Simulation vs.\ real data.} PhysVLMM is best on every quantity in both blocks; in simulation the joint is estimated from a controlled 20\,N pull, and the drawer direction of GAPartNet is further off than a random one (median 60$^\circ$).
\end{itemize}

\paragraph{Run time.} \Cref{tab:published} gives the time per case on one thread of the laptop CPU; RigPI is timed at the 2\,ms step of its runs (\timeRigN{} cases). Within a PhysVLMM case, the mechanism and axis take \timeOursJoint\,ms, the physical belief \timeOursDyn\,ms and the response \timeOursResp\,ms per second of action (50 rollouts); a query takes \timeOursQuery\,ms and a collision check \timeOursColl\,ms per part.

\subsection{Ablation study}\label{sec:ablation}
\begin{table*}[t]\centering\small\setlength{\tabcolsep}{4pt}
\caption{Ablations (same inputs). --: cannot answer. Force, query: 140 untouched parts (17 homes); final position: CRPS of the final joint position of 253 held-out pulls (18 homes). $\Delta$: variant minus PhysVLMM (positive: PhysVLMM better). Control: mean of the home's measured parts of the type.}\label{tab:ablation}
\resizebox{\textwidth}{!}{\begin{tabular}{llccccccl}\toprule
variant & description & physics fit & distribution & category prior & look-alike prior & force\,$\downarrow$ & query\,$\downarrow$ & final position\,$\downarrow$ ($\Delta$ to PhysVLMM) \\\midrule
B1 & joint only, no physics & \ding{55} & \ding{55} & \ding{55} & \ding{55} & -- & -- & \textbf{0.164} ($-$0.151 [$-$0.236, $-$0.086]) \\
B0 & category prior only & \ding{55} & \ding{51} & \ding{51} & \ding{55} & 56\,\% & 0.184 & 0.253 ($-$0.062 [$-$0.124, $-$0.011]) \\
B2$_{\text{pt}}$ & physics fit, point estimate & \ding{51} & \ding{55} & \ding{55} & \ding{55} & -- & -- & 0.460 (+0.145 [+0.101, +0.186]) \\
B2 & physics fit, posterior & \ding{51} & \ding{51} & \ding{55} & \ding{55} & -- & -- & 0.270 ($-$0.045 [$-$0.102, +0.004]) \\
B3 & B2 + category prior & \ding{51} & \ding{51} & \ding{51} & \ding{55} & 56\,\% & 0.184 & 0.285 ($-$0.030 [$-$0.059, $-$0.006]) \\
\textbf{PhysVLMM} & B3 + look-alike prior & \ding{51} & \ding{51} & \ding{51} & \ding{51} & \textbf{42}\,\% & 0.145 & 0.315 \\
\midrule
control & mean of the home's measured parts & \ding{55} & \ding{51} & \ding{55} & \ding{55} & \textbf{42}\,\% & \textbf{0.143} & -- \\
\bottomrule\end{tabular}}\end{table*}

\begin{table*}[t]\centering\small\setlength{\tabcolsep}{4pt}
\caption{Simulation ablations, same inputs and parts. Final position: CRPS of the final joint position under fixed 3--12\,N pulls (12320 pulls, the RigPI set), $\Delta$ = variant minus PhysVLMM (positive: PhysVLMM better) with 95\,\% CI over layouts. Force and query: look-alike rule of the paper on the rendered part points; true groups in brackets (diagnostic).}\label{tab:sim-ablation}
\begin{tabular}{llccl}\toprule
variant & description & force\,$\downarrow$ & query Brier\,$\downarrow$ & final position\,$\downarrow$ ($\Delta$ to PhysVLMM) \\\midrule
B1 & joint only, no physics & -- & -- & 0.743 (+0.682 [+0.675, +0.688]) \\
B0 & category prior only & 58\,\% & 0.183 & 0.131 (+0.070 [+0.066, +0.074]) \\
B2$_{\text{pt}}$ & physics fit, point estimate & -- & -- & 0.208 (+0.146 [+0.145, +0.149]) \\
B2 & physics fit, posterior & -- & -- & 0.169 (+0.107 [+0.103, +0.110]) \\
B3 & B2 + category prior & 58\,\% & 0.183 & 0.064 (+0.003 [+0.002, +0.003]) \\
\textbf{PhysVLMM} & B3 + look-alike prior & 50\,\% (41\,\%) & 0.165 (0.133) & 0.062 \\
\midrule
control & mean of the kitchen's measured parts & 53\,\% & 0.173 & -- \\
\bottomrule\end{tabular}\end{table*}

\begin{figure*}[t]\centering
\includegraphics[width=\linewidth]{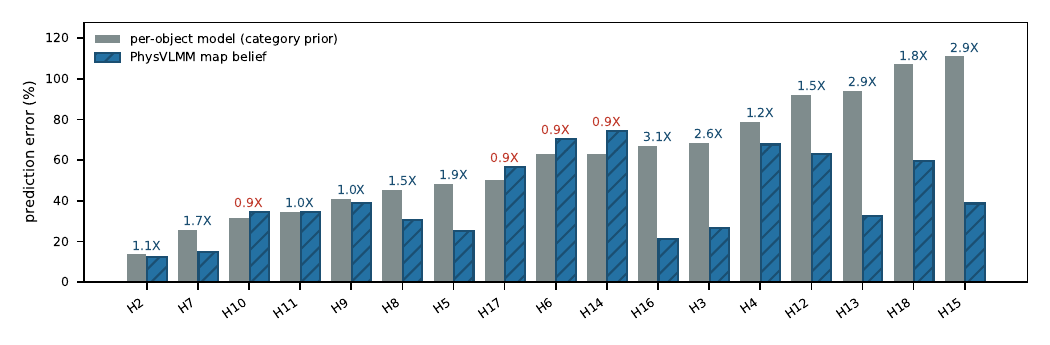}
\caption{Force error of untouched parts per home, per-object prior vs.\ PhysVLMM; labels: improvement factor.}\label{fig:homes}
\end{figure*}
\Cref{tab:ablation} switches the components of PhysVLMM on one at a time, with the same inputs; B2 fits the pulled part under a broad zero-mean prior, without other parts.
\begin{itemize}
\item \emph{Look-alike prior} (B3 $\to$ PhysVLMM). For a part that has not been touched, both the force and the query improve (log-error difference \abFCI; Brier difference \abQCI), and the map belief commits to an answer more often (\abQAnsOurs\,\% of the queries instead of \abQAnsCat\,\%; \abQFSOurs\,\% of its ``yes'' answers are false-safe). This is the gain that Propositions~\ref{prop:passive} and~\ref{prop:measured} leave open: evidence from other measured parts. The control, which averages every measured part of the same type in the home, does as well (log-error difference \abFCIHome; Brier difference \abQCIHome), so in these homes the look-alike groups add nothing measurable over the whole home.
\item \emph{Per home and per force level.} The error of PhysVLMM is lower in \homesBetter{} of \nHomesPred{} homes, by at least 5\,\% in \homesClear{} (\cref{fig:homes}). Per force (\cref{tab:q4}), the per-object model never commits at 3 and 5\,N, while the map belief commits mostly at 3 and 8\,N. Look-alikes from the map's semantic features instead of the geometry make no significant difference (\semCommonN{} parts; log-error difference \semCommonCI).
\item \emph{Posterior} (B2$_{\text{pt}}$ $\to$ B2). Retaining the unidentifiable direction of a one-way pull lowers the final-position error; the whole-trajectory error rises (\cref{tab:t2}).
\item \emph{Priors for the final position} (B2 $\to$ B3 $\to$ PhysVLMM). Neither prior improves the final position of a pulled part, for which its own data dominate; over the whole trajectory, both priors lower the error (\cref{tab:t2}).
\item \emph{No physics} (B1). The mechanism alone gives the smallest final-position error, for the reason below.
\end{itemize}

\begin{figure}[t]\centering
\includegraphics[width=\linewidth]{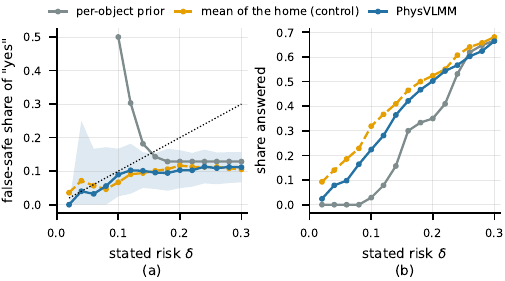}
\caption{Query ``will $F$\,N keep it moving?'' on untouched parts, swept over the stated risk $\delta$. (a) Share of ``yes'' answers that were false-safe (band: 95\,\% CI over homes for PhysVLMM; dotted: $\delta$). (b) Share of queries answered.}\label{fig:risk}
\end{figure}
\paragraph{Stated risk.} A planner sets the risk $\delta$ of a false-safe answer. Swept from 0.02 to 0.30 over the \riskN{} queries of the untouched parts (\cref{fig:risk}a), the false-safe rate of PhysVLMM's ``yes'' answers stays at or below $\delta$, whereas the per-object prior exceeds it where it first answers ($\delta$ = \riskPriorLo--\riskPriorHi). The mean of the home also stays close to $\delta$ and answers slightly more often (\cref{fig:risk}b), as for the force.

\paragraph{Final position.}\label{sec:b1}
B1 predicts that a part always runs to its stop. People pulled most parts to the stop (\spReachShare\,\% of pulls end within 20\,\% of it), so this guess is usually right. On the pulls that stopped short, B1 and the point estimate B2$_{\text{pt}}$ are the worst methods and B2 the best, and PhysVLMM does not significantly beat B1 (B1 minus PhysVLMM \spShortCI; \cref{fig:results}b). B1 also cannot give the force of an untouched part, answer the force query or predict timing. Testing the motion response properly therefore needs fixed actions, which Hoi! does not contain; with fixed pulls in simulation, the order reverses (\cref{tab:sim-ablation}).

\paragraph{Ablations in simulation.} \Cref{tab:sim-ablation} repeats the ablations under fixed pulls. Unlike on Hoi!, the look-alike prior also improves the final position of a pulled part over B3, and the force of an untouched part over the mean of the kitchen (log-error difference \simForceHomeCI). The physics of look-alike parts is, however, drawn from a group model fitted on Hoi!, so this gain shows what the grouping delivers when that model holds, not that it holds in real homes.

\subsection{Kinematic belief and collision}\label{sec:kin}
\paragraph{Mechanism and axis.} Over all \tOneN{} pulls of the \tOneHomes{} homes (\fbOursAcc\,\% on the \fbN{} held-out pulls of \cref{tab:published}), the mechanism belief identifies the joint type in \tOneAcc\,\% of cases (doors \tOneAccDoor\,\%, drawers \tOneAccDrawer\,\%), with a Brier score of \tOneBrier{} against \tOneBrierPrior{} for the prior alone. The median axis errors are \tOneAxDrawer$^\circ$ for drawers, and \tOneAxDoor$^\circ$ and \tOnePosDoor\,cm for door hinges. When the type is wrong, the belief still gives the wrong type a mean probability of \tOneConfWrong, because a person's grip on a handle is not rigid and the gripper path then fits the wrong joint well.

\paragraph{Collision.} Every part is swept along its joint and checked against the scanned scene. The opening a person actually reached is free of collision, so a collision predicted before 90\,\% of it is a false alarm. On the \colN{} parts of the held-out homes, PhysVLMM's swept part points raise false alarms on \colOurs\,\% of parts, against \colBox\,\% when the part is approximated by its bounding box (ablation). Hoi! contains no collisions, so missed collisions are measured only in simulation: of the \simCollBlocked{} blocked parts, PhysVLMM misses \simMissed\,\% and the bounding box \simMissedBox\,\%, with false alarms on \simFA\,\% of the free parts. The collision module is therefore not yet reliable enough for a safety guarantee.

\section{Conclusion}\label{sec:conclusion}
This paper presented PhysVLMM, a method that maintains a physical belief for every articulated part of a Vision-Language-Motion Map. Its three steps follow from two identifiability propositions: it keeps the direction that a one-way pull leaves unidentified as a posterior, passes each measured pull in closed form to the untouched look-alike parts, and answers a planner's queries at a stated risk or abstains. On real doors and drawers in \nHomes{} homes, it identifies the joint type more often than learned single-view methods and predicts the response to a pull with lower error than differentiable simulation, at four orders of magnitude lower cost. In these homes, pooling all measured parts of a home does as well as the look-alike groups, and simpler variants predict where a pulled part stops at least as well. Future work will evaluate lids, windows and valves, fixed actions of a real robot with collisions and partial openings, perceived instead of annotated parts, and buildings other than homes.

{
    \small
    \bibliographystyle{ieeenat_fullname}
    \bibliography{main}
}

\twocolumn[\centering\Large\textbf{Appendix}\par\vspace{1.0em}]  %
\appendix
\section{Method details}\label{sec:s-method}
\paragraph{Proof of Proposition~\ref{prop:passive}.} For any $\theta$, the choice $\tau_a(t) + \tau_{\mathrm{ext}}(t) = \varphi(t)^{\!\top}\theta$ is piecewise continuous (it jumps by $2\tau_c$ where $\dot q$ changes sign) and reproduces $q(t)$ exactly by~\eqref{eq:linear}. With a prior on the force, the trajectory determines the force given $\theta$, which gives the stated likelihood; the change of variables from force to trajectory adds a $\theta$-dependent Jacobian.

\paragraph{Proof of Proposition~\ref{prop:measured}.} Without $\tau_{\mathrm{ext}}$ and contact, \eqref{eq:linear} is the linear system $\Phi\theta_i = \boldsymbol\tau_a$, which has a unique solution if and only if $\Phi$ has full column rank. For constant $\tau_a$, $\Phi\theta_i = \tau_a\mathbf 1 = \Phi(\tau_a e_5)$, because the last column of $\Phi$ is $\mathbf 1$; hence $\theta_i - \tau_a e_5$ lies in the null space of $\Phi$. If $\mathrm{sgn}\,\dot q$ is constant, the third column is a multiple of the last. The likelihood of part $i$'s data depends on $\theta_i$ only, which gives the last statement.

\paragraph{Contact and belief update.} The contact reactions $\tau_{\mathrm{con}}$ in~\eqref{eq:joint} satisfy the complementarity
\begin{equation}
0 \le \lambda_n \;\perp\; d(q) \ge 0
\label{eq:comp}
\end{equation}
between the normal contact force $\lambda_n$ and the signed distance $d(q)$ of the link to obstacles and joint stops. An unmeasured external force $z_t$ is a latent input and is marginalised in the transition,
\begin{equation}
\begin{aligned}p(\mathbf q_t \mid \mathbf q_{t-1}, u_{t-1}, m, \theta) &\\ = \!\int\! p(\mathbf q_t \mid \mathbf q_{t-1}, &\, u_{t-1}, z, m, \theta)\, p(z)\, dz,\end{aligned}
\label{eq:marg}
\end{equation}
and the belief~\eqref{eq:belief} is updated recursively with the measurement $Y_t$:
\begin{equation}
\begin{aligned}\mathcal B_{i,t} \propto{}& p(Y_t \mid \mathbf q_t, m, \theta)\\ &\times\!\int\! p(\mathbf q_t \mid \mathbf q_{t-1}, u_{t-1}, m, \theta)\, \mathcal B_{i,t-1}\, d\mathbf q_{t-1}.\end{aligned}
\label{eq:filter}
\end{equation}
Because $\theta_i$ is static, this reduces to a batch posterior over each measured pull. The implementation takes $p(z)$ zero-mean Gaussian and white in time; in the linear model~\eqref{eq:linear} the marginal is then the Gaussian noise of the regression~\eqref{eq:post}, with variance $\sigma^2_{\mathrm{eff}}$. The rollouts of~\eqref{eq:roll} assume no unmeasured force during the predicted action.

\begin{figure}[t]\centering
\begin{tikzpicture}[scale=1.45, every node/.style={font=\scriptsize}]
\fill[gray!22] (-0.1,-0.32) rectangle (2.5,0); \node at (1.3,-0.16) {cabinet (fixed); closed door at $q=0$};
\fill[blue!8] (0,0) -- (2.4,0) arc (0:40:2.4) -- cycle; \node[blue!60!black] at (1.85,0.42) {swept volume};
\draw[dotted, gray] (40:2.4) arc (40:90:2.4);
\fill[orange!25, draw=orange!70!black] (0.95,2.17) rectangle (1.95,2.7); \node[align=center] at (1.45,2.44) {obstacle (scan)};
\draw[dashed, red!70!black] (0,0) -- (65.5:2.4); \node[red!70!black, left, align=right] at (65.5:1.3) {first contact\\ at $q_{\mathrm{free}}$};
\draw[very thick, blue!70!black] (0,0) -- (40:2.4); \node[blue!70!black, rotate=40, above] at (40:1.0) {door link $D$};
\filldraw[blue!70!black] (40:2.05) circle (1.2pt); \node[below right] at (40:2.05) {handle $\mathbf r$};
\draw[-{Latex[length=2mm]}, thick, green!50!black] (40:2.05) -- ++(130:0.6); \node[green!40!black, right] at ($(40:2.05)+(130:0.45)+(0.08,0.08)$) {pull $w$};
\draw[-{Latex[length=1.5mm]}] (0.55,0) arc (0:40:0.55); \node at (18:0.75) {$q$};
\filldraw (0,0) circle (1.4pt); \node[above left, align=right] at (0,0) {hinge $H$,\\ axis $\hat a_i$ ($\odot$)};
\end{tikzpicture}
\caption{Worked door model (plan view). The door link rotates about the hinge axis by the joint coordinate $q$~\eqref{eq:poe}; a wrench $w$ at the handle drives it through~\eqref{eq:joint}. The swept volume of the link is tested against the scanned scene; the first contact opening $q_{\mathrm{free}}$ bounds the collision-free motion.}\label{fig:door}
\end{figure}
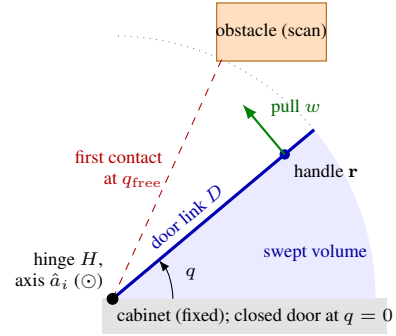
\begin{table}[t]\centering\small\setlength{\tabcolsep}{3pt}
\caption{Identifiability of the door dynamics $\theta = (I, b, \tau_c, k, \tau_0)$ from a \emph{measured} applied torque (simulated single revolute joint; rank of the Fisher information out of 5). Without a measured torque the rank is 0 (\cref{prop:passive}).}\label{tab:fisher}
\begin{tabular}{lcc}\toprule
applied torque profile & rank & condition no. \\\midrule
constant force, opening only & 3/5 & $\infty$ \\
ramped force, opening only & 4/5 & $\infty$ \\
pull, then release (closer returns the door) & 5/5 & 17 \\
two force levels, then release & 5/5 & 13 \\
\bottomrule\end{tabular}\end{table}

\paragraph{VLMM layer (M1).} For a part with label $\ell_i$, the semantic feature is the CLIP text embedding $\mathbf f_i = \mathrm{CLIP}(\text{``a photo of a } \ell_i\text{''})$. The movability prior $\rho_i$ comes from VLMM's per-label language-model query. The observed motion is $o_i = s_i/\sigma_p$, where $s_i$ is the tracked stroke of the link and $\sigma_p = 1$\,cm the pose noise, and its reliability $\kappa_i$ is the confidence of the mechanism belief below. The motion class $y_i$ and its confidence $c_i$ follow VLMM's fusion rule, in which reliable observed motion overrides the prior~\cite{ghosh2026vlmm}. Because $o_i$ and $\kappa_i$ are computed from the same observations as the link motion, they parameterise the prior or classify the element; they are never multiplied in as a second, independent likelihood.

\paragraph{Mechanism hypotheses (M3).} From the link poses $(\mathbf c_t, R_t)$ observed during a motion (M2), two models are fitted.
\begin{itemize}
\item \emph{Prismatic:} a straight path with no rotation. The residuals are the perpendicular distances to the principal line and the rotation angles $\|\log(R_t R_0^{\!\top})\|$.
\item \emph{Revolute:} a rigid rotation about a fixed axis, with
\begin{equation}
\mathbf c_t = \mathbf a + \mathrm{Rot}(\hat a, \phi_t)(\mathbf c_0 - \mathbf a),
\label{eq:rigid}
\end{equation}
where $\hat a$ is the angle-weighted mean of the relative rotation vectors, $\phi_t$ their projection on $\hat a$, and $\mathbf a$ the least-squares solution of $(I - \mathrm{Rot}(\hat a, \phi_t))\mathbf a = \mathbf c_t - \mathrm{Rot}(\hat a, \phi_t)\mathbf c_0$.
\end{itemize}
Each model gets a profile log-likelihood with the residual variances estimated from the data, and a BIC. In principle, a few degrees of motion suffice for the joint type: in an idealised test on one rotation measurement with 1$^\circ$ of pose noise, an opening of 1$^\circ$, 2$^\circ$ or 5$^\circ$ makes a revolute joint preferred over a prismatic one with probability \kinOne, \kinTwo{} or \kinFive. Since a person's grip is not rigid, the evidence is combined with a position-only test (line vs.\ circular arc), the rotation magnitude $\psi_i$ and the path length $s_i$ into a mechanism belief:
\begin{align}
P(m_i = \text{rev} \mid \text{motion}) &= \sigma\big(\beta^{\!\top} z_i\big), \label{eq:mech}\\
z_i &= \big(1, \ell(\Delta\mathrm{BIC}_i), \ell(\Delta\mathrm{BIC}^{\mathrm{arc}}_i), \nonumber\\ &\qquad \log(1+\psi_i), \log s_i\big), \nonumber
\end{align}
with $\ell(v) = \mathrm{sgn}(v)\log(1+|v|)$ and $\beta$ fitted by L2-regularised logistic regression (unit weight) on other homes. The axis estimate is the principal direction of the path for a prismatic joint, and the plane normal and centre of the fitted arc for a revolute joint.

\paragraph{Effort from a measured pull (M4).} A pull provides the handle force $f(t)$, the gripper position $\mathbf c(t)$ and the grip signal $\gamma(t)$. The grasped phase and the moving set are
\begin{align}
\mathcal T_w &= \{t : \gamma(t) > \gamma_0\}, \\
\Omega_w &= \{t \in [\min\mathcal T_w, \max\mathcal T_w] : \|\dot{\mathbf c}(t)\| > v_0\},
\end{align}
and the effort of the part is
\begin{equation}
x_i = \frac{1}{|W_i|}\sum_{w \in W_i} \log \operatorname*{median}_{t \in \Omega_w} \|f(t)\|,
\label{eq:extract}
\end{equation}
with $\gamma_0 = 5$ and $v_0 = 2$\,cm/s.

\paragraph{Look-alike groups (M4).} Grouping uses only what is perceived before any part moves: the part's description and its scan points $\mathcal P_i$. The type is prismatic if the description names a drawer or a sliding or pull-out part, and revolute otherwise. The front of a part is the dominant plane of $\mathcal P_i$ (RANSAC), with unit normal $\hat n_i$ and centre $\mathbf s_i$. Two parts of the same home and type are look-alikes by geometry, by semantics, or by either:
\begin{align}
A^{\mathrm{geo}}_{ij} &= \mathbf 1\big[|\hat n_i\cdot\hat n_j| > \cos\epsilon_a,\ \|\mathbf s_i - \mathbf s_j\| < d_0\big], \label{eq:adj}\\
A^{\mathrm{sem}}_{ij} &= \mathbf 1\big[\mathbf f_i^{\!\top}\mathbf f_j > c_0\big], \qquad A^{\cup} = A^{\mathrm{geo}} \vee A^{\mathrm{sem}}, \nonumber
\end{align}
with $\cos\epsilon_a = 0.98$, $d_0 = 1.5$\,m and $c_0 = 0.9$. Groups $g(i)$ are the connected components of the chosen graph (single linkage); the geometric graph is the default. Parts with fewer than 50 scan points or without a description are left out of the look-alike groups and of the hyper-parameter fits. The joint annotations are not used for grouping.

\paragraph{Map belief over the effort (M4).} With the model~\eqref{eq:model}, specialising~\eqref{eq:filter} with natural parameters initialised to $(\lambda_g, \eta_g) = (1/\sigma_0^2, \mu_c/\sigma_0^2)$, every measured effort $x_i$ in group $g$ updates
\begin{equation}
\lambda_g \leftarrow \lambda_g + \tfrac{1}{\omega^2}, \quad \eta_g \leftarrow \eta_g + \tfrac{x_i}{\omega^2},
\label{eq:update}
\end{equation}
which gives the predictive~\eqref{eq:pred}. A repeat pull of a part replaces its $x_i$ (the mean over its pulls) and does not add a second update. The category mean $\mu_c$ is the mean effort of the type in other homes, and $(\sigma_0^2, \omega^2)$ maximise the marginal likelihood of the groups of other homes over a logarithmic grid of $14 \times 14$ values ($\sigma_0^2$ from 0.005 to 1.5, $\omega^2$ from 0.005 to 1.0; for the outlier-aware variant, $\pi \in \{0.6, 0.7, 0.8, 0.9, 0.95\}$). The efforts $y_g$ of a group of $n_g$ parts are Gaussian with covariance $\omega^2 I + \sigma_0^2\mathbf 1\mathbf 1^{\!\top}$, and with $d = y_g - \mu_c\mathbf 1$ the Sherman--Morrison identity gives
\begin{multline}
-2\log p(y_g) = n_g\log 2\pi + (n_g-1)\log\omega^2 + \log(\omega^2 + n_g\sigma_0^2) \\
+ \frac{1}{\omega^2}\Big(\|d\|^2 - \frac{\sigma_0^2(\mathbf 1^{\!\top} d)^2}{\omega^2 + n_g\sigma_0^2}\Big).
\label{eq:ml}
\end{multline}

\paragraph{Map belief over the dynamics (M4).} For the pulled part, $\theta_i$ gets a Gaussian prior from the map, built per parameter $p$: the category median and variance $(\mu_{c,p}, s^2_{c,p})$ of parameters identified from two-direction pulls of parts in other homes (in simulation, from all five pulls), shrunk towards the measured look-alikes like~\eqref{eq:pred} with within- and between-group variances estimated by moments from the look-alike groups of other homes, giving $(\mathbf m_0, V_0)$. The pull's regressors~\eqref{eq:linear} are computed in the joint coordinate of the estimated axis, from the samples that move ($|\dot q| > 1$\,cm/s or 0.03\,rad/s). Each pull is first placed on its true 100\,Hz time base. Samples within 10\,\% of the stroke from either end are excluded, because there the end-stop contact of~\eqref{eq:comp} is active, together with the 15 samples on either side of that zone, over which the 31-sample smoothing spreads an end-stop impact; samples inside a recording gap and their 15-sample neighbourhood are excluded as well. This gives the Gaussian posterior
\begin{align}
V &= \big(V_0^{-1} + \Phi^{\!\top}\Phi/\sigma^2_{\mathrm{eff}}\big)^{-1}, \label{eq:post}\\
\mathbf m &= V\big(V_0^{-1}\mathbf m_0 + \Phi^{\!\top}\boldsymbol\tau_a/\sigma^2_{\mathrm{eff}}\big), \nonumber
\end{align}
where $\sigma^2_{\mathrm{eff}}$ is the residual variance multiplied by a fixed factor of 15, about half the 31-sample smoothing window. The factor is not estimated from the data: it widens the posterior to absorb the error that noise in the smoothed derivatives causes in the regressors (in a synthetic test with the pose and force noise of Hoi!, the fitted inertia is otherwise about 30\,\% too low). Physical bounds ($I \ge I_{\min}$, $b, \tau_c \ge 0$) are applied by clamping each sample in the rollout. Along the unidentifiable direction of a one-way pull (Proposition~\ref{prop:measured}), the posterior precision equals that of the prior, so this direction is not resolved arbitrarily.

\paragraph{Response (M5).} For an action $u$ with generalised force $\tau_a(t)$, posterior samples $\theta^{(s)}$ are propagated through~\eqref{eq:joint} with a semi-implicit step of $\Delta t = 0.01$\,s:
\begin{align}
\dot q_{n+1} &= \frac{\dot q_n + \Delta t\,\big(\tau_{a,n} - \tau_c\,\mathrm{sgn}\,\dot q_n - k q_n - \tau_0\big)/I}{1 + \Delta t\, b/I}, \label{eq:roll}\\
q_{n+1} &= q_n + \Delta t\,\dot q_{n+1}. \nonumber
\end{align}
Stiction keeps $\dot q = 0$ while $|\tau_{a,n} - k q_n - \tau_0| \le \tau_c$, and the joint stops at the range reached in the training pull, extended by 5\,\% at each end (the limits are not observed beyond it), the active case of~\eqref{eq:comp}. The samples form the motion part of~\eqref{eq:response}.

For contact, the link's point set $\mathcal P_i$ is reduced to its dominant plane by RANSAC, the handle points just in front of it are attached to it, and it is swept along the joint by~\eqref{eq:poe}:
\begin{equation}
\mathcal P_i(q) = \mathbf a_i + \mathrm{Rot}(\hat a_i, q)(\mathcal P_i - \mathbf a_i) \quad\text{or}\quad \mathcal P_i + q\,\hat a_i .
\end{equation}
The contact opening $q_{\mathrm{free}}$ (\cref{fig:door}) is the first $q$ at which $\min_{\mathbf x \in \mathcal P_i(q)} \operatorname{dist}(\mathbf x, \mathcal O_i) < d_c$, \ie $d(q) < d_c$ in~\eqref{eq:comp}. The obstacle set $\mathcal O_i$ contains the scanned points more than 3\,cm in front of the part's face, excluding fixtures within 4\,cm of its outline, within reach of the swept part (its extent plus 0.9\,m for drawers or 0.2\,m for doors), and $d_c = 1$\,cm (1.5 times the scan point spacing). A collision predicted before 90\,\% of the opening a person reached counts as a false alarm.

\paragraph{Queries (M6).}
\begin{itemize}
\item \emph{Movable} and \emph{moved} are read from $y_i$; \emph{direction} from the mechanism belief and the axis.
\item \emph{Force:} with the predictive~\eqref{eq:pred}, $(\hat x_j, v_j)$, the risk test~\eqref{eq:risk} has the closed form
\begin{equation}
P(\bar F_j \le F \mid \mathcal H_t) = \Phi_{\mathcal N}\Big(\frac{\log F - \hat x_j}{\sqrt{v_j}}\Big),
\label{eq:answer}
\end{equation}
answered ``yes'' at $\ge 1 - \delta$, ``no'' at $\le \delta$, and ``insufficient evidence'' otherwise.
\item \emph{Collision} before the target opening $q^\ast$ is reported if $q_{\mathrm{free}} < q^\ast$.
\end{itemize}

\paragraph{Limit of sharing.}
\begin{proposition}[Limit of sharing]\label{prop:limit}
Under~\eqref{eq:model}, the predictive variance $v_j(n) = (1/\sigma_0^2 + n/\omega^2)^{-1} + \omega^2$ decreases with the number $n$ of measured look-alikes towards $\omega^2$. Hence, for any $n$, the probability that an untouched part's effort lies within a log-tolerance $\varepsilon$ of its prediction is at most $2\Phi_{\mathcal N}(\varepsilon/\omega) - 1$.
\end{proposition}
\begin{proof}
$x_j - \hat x_j \sim \mathcal N(0, v_j(n))$, and $v_j(n) \ge \omega^2$; for a zero-mean Gaussian, $P(|x_j - \hat x_j| \le \varepsilon)$ decreases with the variance.
\end{proof}
On Hoi!, the fitted spread within a group is a factor of \tauFactorFit{} ($e^{\omega}$), so the model implies that about \boundShare\,\% of untouched parts are predicted within a factor of 1.25 when the group effort is known exactly (prior alone: \priorShare\,\%). On the held-out parts, PhysVLMM reaches \obsShareOurs\,\%, the mean of the home \obsShareHome\,\% and the prior \obsSharePrior\,\%.

\paragraph{Complexity.} Grouping is $O(n^2)$ once per map. An effort update or query is $O(1)$. A dynamics posterior is $O(N)$ in the number of samples of the pull. A response costs $S$ rollouts of the action's length, and a contact query one sweep with a $k$-d tree. \Cref{alg:update,alg:query} summarise the procedure.

\section{Processing details}\label{sec:s-proc}
\paragraph{Files used.} Per recording session, the inputs are the gripper force/torque stream (gravity-compensated and filtered by the dataset), the gripper trigger stream, the Aria SLAM trajectory aligned to the scan frame, the session time-synchronisation file, the confirmed interaction windows and the per-home axis annotations of the laser scans. Two homes have no gripper force data. In one home, one of two sessions has no force file. One window file in the release is empty.

\paragraph{Motion and force.} Gripper speed is computed from the trajectory resampled at 10\,Hz. A window is used if the gripper moved enough and stayed near the part, measured without the joint annotation: the range of the gripper path along its own principal direction is at least 5\,cm (drawers) or the range of its rotation about its own principal rotation direction is at least 10$^\circ$ (doors), with the type from the description, and the path's median distance to the part's scan points is below 1\,m. The joint annotations give the travel along the annotated axis that is reported for description only. The running force of a pull is the median handle force while the part moves, which also contains some grip and sideways force, and the effort $x_i$ is the mean of its logarithm over the part's opening and closing pulls. For the dynamics, the force is rotated from the sensor into the gripper frame with the rotation that best aligns force and gripper velocity over the pulls of the other parts in the same session.

\paragraph{Speed-normalised force.} The force is projected on its first principal axis in the sensor frame, during motion, with the sign chosen so that the median is positive. The model $F = f_0 + b v + m a$ is fitted by least squares, with $v$ smoothed over 0.3\,s and $a = \dot v$, and $f_{\mathrm{ref}} = f_0 + b\,v_{\mathrm{ref}}$ with $v_{\mathrm{ref}}$ the median moving speed over all windows (\vRef\,m/s). Windows with less than 1\,s of motion, a speed range under 0.05\,m/s, or $f_{\mathrm{ref}} \le 0$ are excluded.

\section{Audit of the window-to-part labels}\label{sec:s-audit}
For drawers, all drawers of a home are ranked by the mean distance of the grasped gripper path to their axis line. For doors, all hinges are ranked by the relative spread of the path's distance to the hinge line, since a handle moves on a circle about its own hinge. The assigned part ranks first in \mapPriFirst\,\% of drawer windows (chance \mapPriChance\,\%) and \mapRevFirst\,\% of door windows (chance \mapRevChance\,\%). Most disagreements are exact or near ties between coaxial or parallel neighbours, such as the doors of one fridge. Removing the \mapClear{} windows that clearly disagree leaves the look-alike ratio at \ratioCleanBoth{} and gives, for outlier-aware pooling, an NLL gain over the category prior of \dPoolPriorClean{} (CI \ciPoolPriorClean) instead of \dPoolPrior{}, and outlier-aware minus hierarchical pooling of \dPoolHierClean{} (CI \ciPoolHierClean).

\section{Known-answer tests}\label{sec:s-tests}
Every component was checked on synthetic cases with known answers before the real-data evaluation:
\begin{itemize}
\item a noisy door arc yields its hinge within 0.5$^\circ$ and 0.02\,cm;
\item a noisy drawer path yields its axis within 0.01$^\circ$;
\item a two-way pull of a simulated drawer with $\theta = (4, 6, 3, 2, 1.5)$, run through the same pipeline as the real data, yields $(4.18, 6.81, 2.64, 1.89, 1.71)$, and the held-out rollout errs by 0.07\,cm; with a 0.5\,s recording gap during the opening, the held-out rollout errs by 0.29\,cm;
\item a pull that reaches its stop within 1\,s leaves too few samples away from the stop to identify $\theta$ (a known limit; its rollout is still reproduced, 0.17\,cm);
\item a one-way pull has a rank-deficient regressor (singular value ratio $10^{-17}$), as Proposition~\ref{prop:measured} states;
\item the closed form~\eqref{eq:answer} matches Monte Carlo to $1.1\times10^{-3}$;
\item the closed-form member density and the enumerated mixture of the outlier-aware pooling match a brute-force evaluation with generic multivariate normal densities to $7\times10^{-15}$, and every predictive density integrates to one on its grid;
\item a door swept towards a wall at 60$^\circ$ stops at 58$^\circ$.
\end{itemize}
The tests found three errors, which were fixed before the reported results:
\begin{itemize}
\item end-stop contact samples in the dynamics regression corrupted $\theta$ (inertia 0.5 instead of 4);
\item the 1\,cm contact distance had not been applied;
\item the 31-sample smoothing spread end-stop impacts into the samples kept for the regression, and pulls with recording gaps were treated as evenly sampled; both appeared once the test ran on the full pipeline.
\end{itemize}
The collision module was further revised on two development homes only; \cref{tab:t3} reports the held-out homes:
\begin{itemize}
\item the part's own handle was treated as an obstacle;
\item annotated part points included background behind the panel;
\item fixtures hugging a door's outline triggered contact.
\end{itemize}

\section{Adapters for the published methods}\label{sec:s-baselines}
\paragraph{FlowBot3D.} Released code and weights (\texttt{r-pad/flowbot3d}, commit 5189de2, checkpoint \texttt{model.ckpt}). Input: the scan of the setup in which the part was annotated, cropped to the part's extent plus 0.4\,m, 1200 points of which up to 40\,\% are the part (oracle part mask), in metres, centred. FlowBot3D predicts the motion direction of every point. Joint type and axis are obtained from the predicted flow on the part: a translation fits one direction, a rotation fits $f_i = \omega \times (\mathbf p_i - \mathbf c)$ by linear least squares; the type is decided by a logistic rule on the spread of flow magnitudes and the distance of the fitted axis, fitted on other homes. Because FlowBot3D was trained on objects in a fixed orientation, six input orientations were tried and the best on the two development homes (\fbOrient) is reported on the held-out homes.
\paragraph{GAPartNet.} Released code and checkpoint (\texttt{PKU-EPIC/GAPartNet}, commit c8d4ad2, \texttt{all\_best.ckpt}), run on CUDA 13 with two compatibility edits to its \texttt{epic\_ops} dependency (the newer PyTorch dispatch call and an explicit Thrust include). Input: a coloured crop of the scan around the part (0.2\,m margin, floor removed, $\pm$0.5\,m about the part's face, chosen over the FlowBot3D crop on the development homes), in a virtual camera frame looking at the part (OpenCV convention), 20\,000 points normalised to the unit ball by GAPartNet's own rule. The released test step reads ground-truth labels when they are given; GAPartNet's own semantic predictions are passed in the label slot and non-zero dummies in the NPCS slot, which reproduces its label-free path. The joint type is the GAPart class on the part (hinge: revolute, slider: prismatic; no actionable class on at least 5\,\% of the part's points: no answer, scored as wrong). A drawer's direction is the box's outward face normal, as in GAPartNet's manipulation code.
\paragraph{RigPI.} Released code (\texttt{Eigenoperator/RigPI}, commit 19b0c79) with its Newton/Featherstone differentiable rollout and optimiser (120 iterations). Per case, a one-joint URDF places the joint on the estimated axis with the limits of the training pull ($\pm$5\,\%), as for PhysVLMM. Inputs: the link poses of the training pull and the measured handle force at the gripper. RigPI's vision-language prior is not released; its initial values and bounds (0.2--5$\times$) come from the scanned panel (18\,mm wood, 650\,kg/m$^3$). Friction load: $m g$ with a 5\,cm radius for hinges. All rollouts, training and held-out, use 2\,ms sub-steps (the joint-limit spring diverges at 10\,ms when a pull reaches the stop). Parts with fewer than 50 scan points give no panel prior and are not exported. For \rigFail{} of the \rigCases{} exported cases RigPI returns no prediction (a non-finite gradient or objective during the fit), and for \rigNoFit{} more the training pull leaves too few samples away from the stop for the dynamics fit of PhysVLMM; the comparison uses the remaining \rigN{} cases for both methods.

\section{Detailed results}\label{sec:s-details}
\Cref{tab:homes,tab:t2,tab:t3} give the force error per home, the trajectory error of the held-out response and the collision check.
\begin{table}[t]\centering\small\setlength{\tabcolsep}{3pt}
\caption{Test homes (Hoi! dataset) and the typical error (\%) of the force predicted for a part \emph{before it is touched}: per-object model (category prior) vs.\ the PhysVLMM map belief. Parts: parts with at least one look-alike; groups: look-alike groups of at least two parts; pulls: valid pulls of these parts; homes are numbered, not named.}\label{tab:homes}
\resizebox{\linewidth}{!}{\begin{tabular}{lrrrrrr}\toprule
home & parts & groups & pulls & per-object & PhysVLMM & factor \\\midrule
H2 & 2 & 1 & 4 & 14 & 13 & 1.1$\times$ \\
H3 & 8 & 1 & 16 & 69 & 27 & 2.6$\times$ \\
H4 & 11 & 2 & 21 & 79 & 68 & 1.2$\times$ \\
H5 & 5 & 2 & 9 & 49 & 25 & 1.9$\times$ \\
H6 & 5 & 1 & 10 & 63 & 71 & 0.9$\times$ \\
H7 & 6 & 1 & 12 & 26 & 15 & 1.7$\times$ \\
H8 & 4 & 1 & 8 & 45 & 30 & 1.5$\times$ \\
H9 & 14 & 2 & 27 & 41 & 39 & 1.0$\times$ \\
H10 & 17 & 4 & 34 & 31 & 35 & 0.9$\times$ \\
H11 & 15 & 3 & 29 & 35 & 35 & 1.0$\times$ \\
H12 & 9 & 3 & 18 & 92 & 63 & 1.5$\times$ \\
H13 & 7 & 2 & 14 & 94 & 33 & 2.9$\times$ \\
H14 & 6 & 2 & 10 & 63 & 74 & 0.9$\times$ \\
H15 & 8 & 1 & 16 & 111 & 39 & 2.9$\times$ \\
H16 & 4 & 1 & 8 & 67 & 21 & 3.1$\times$ \\
H17 & 15 & 2 & 27 & 50 & 57 & 0.9$\times$ \\
H18 & 4 & 2 & 8 & 107 & 60 & 1.8$\times$ \\
\midrule
all & 140 & & & 56 & 42 & 1.3$\times$ \\
\bottomrule\end{tabular}}\end{table}

\begin{table}[t]\centering\small\setlength{\tabcolsep}{3pt}
\caption{Whole-trajectory error of the held-out response (RMSE as a fraction of the stroke; the same 253 pulls as \cref{tab:ablation}, whose motion column gives the final-position CRPS). --: B1 predicts no trajectory.}\label{tab:t2}
\begin{tabular}{lc}\toprule
predictor & trajectory RMSE\,$\downarrow$ \\\midrule
B1: mechanism only (runs to the stop) & -- \\
B0: category prior & 0.543 \\
B2$_{\text{pt}}$: physics fit, point estimate & 0.468 \\
B2: physics fit, posterior & 0.608 \\
B3: B2 + category prior & 0.502 \\
\textbf{PhysVLMM}: B3 + look-alike prior & 0.488 \\
\bottomrule\end{tabular}\end{table}

\begin{table}[t]\centering\small\setlength{\tabcolsep}{4pt}
\caption{Swept-volume collision check: share of parts for which a collision is predicted before 90\,\% of the opening a person actually reached (false alarm; lower is better). Development homes were used to set the edge tolerance.}\label{tab:t3}
\begin{tabular}{lccc}\toprule
joint axis & held-out homes & doors & drawers \\\midrule
estimated & 0.23 & 0.20 & 0.25 \\
annotated (oracle) & 0.19 & 0.16 & 0.22 \\
\bottomrule\end{tabular}\end{table}

\section{Ablation tables}\label{sec:s-tables}
The tables below give the full results behind the ablations of \cref{sec:ablation}. The predictors are the per-object category prior, the mean of the same home, the mean of the look-alikes, hierarchical pooling (the PhysVLMM effort belief), and outlier-aware pooling, in which each look-alike is a member with probability $\pi$ or an odd part drawn from the category distribution. All rules and metrics were fixed before the respective test was run. With look-alike groups from the scan (\cref{eq:adj}), the ratio $\omega^2/s^2$ of \cref{tab:pool} falls in the STOP range of its fixed rule: look-alikes are less alike than the annotated groups suggested. The held-out test of \cref{tab:pred}, which decides the claims of the paper, still shows a gain over the category prior.
\begin{table}[t]\centering\small\setlength{\tabcolsep}{2.5pt}
\caption{How alike are look-alikes? Within-group variance $\omega^2$ over the variance of all parts $s^2$ (log force; 0 = look-alikes identical, 1 = no more alike than random parts), with 95\,\% bootstrap CI over groups, and the mean squared log-force difference of look-alike pairs vs.\ other pairs in the same home. Fixed rule: GO if $\omega^2/s^2 \le 0.25$, STOP if $\ge 0.5$.}\label{tab:pool}
\resizebox{\linewidth}{!}{\begin{tabular}{llccc}\toprule
force value & data & $\omega^2/s^2$ [CI] & look-alike pairs & other pairs \\\midrule
open + close & all windows & 0.60 [0.41, 0.83] & 0.35 & 0.46 \\
opening & all windows & 0.77 [0.57, 0.99] & 0.55 & 0.58 \\
closing & all windows & 0.49 [0.28, 0.76] & 0.41 & 0.64 \\
\midrule
open + close & suspect windows removed & 0.60 [0.41, 0.81] & 0.37 & 0.58 \\
opening & suspect windows removed & 0.82 [0.62, 1.02] & 0.59 & 0.58 \\
closing & suspect windows removed & 0.44 [0.23, 0.73] & 0.38 & 0.68 \\
\midrule
open + close & speed-normalised & 0.50 [0.34, 0.73] & 0.60 & 1.17 \\
opening & speed-normalised & 0.55 [0.31, 0.84] & 0.58 & 1.00 \\
closing & speed-normalised & 0.48 [0.34, 0.69] & 1.10 & 1.88 \\

\bottomrule\end{tabular}}\end{table}

\begin{table}[t]\centering\small\setlength{\tabcolsep}{3pt}
\caption{Predicting an untouched part's force from the measured forces of the other parts in its home (leave-one-part-out; hyper-parameters fitted on the other homes). NLL and MAE of log force (lower is better); $\times$ = typical multiplicative error; cov.\ = coverage of the 80\,\% interval. Left: running force; right: speed-normalised force.}\label{tab:pred}
\resizebox{\linewidth}{!}{\begin{tabular}{lcccc@{\hskip 8pt}cc}\toprule
& \multicolumn{4}{c}{running force} & \multicolumn{2}{c}{speed-norm.} \\
predictor & NLL & MAE & $\times$ & cov. & NLL & MAE \\\midrule
category prior (per-object) & 0.852 & 0.446 & 1.56 & 0.79 & 1.184 & 0.636 \\
mean of the same home & 0.661 & 0.349 & 1.42 & 0.80 & 0.974 & 0.477 \\
mean of the look-alikes & 0.695 & 0.355 & 1.43 & 0.81 & 1.073 & 0.512 \\
hierarchical pooling & 0.662 & 0.352 & 1.42 & 0.81 & 0.980 & 0.495 \\
outlier-aware pooling & 0.635 & 0.342 & 1.41 & 0.81 & 1.015 & 0.503 \\
\bottomrule\end{tabular}}\end{table}

\begin{table}[t]\centering\small\setlength{\tabcolsep}{3pt}
\caption{Touch one, know many: share of a home's parts the robot must touch before 80\,\% of them are known within a factor of 1.25 (touched, or predicted within that factor). Mean over homes; random = 20 random touch orders; greedy = largest summed variance reduction.}\label{tab:touch}
\begin{tabular}{lcc}\toprule
predictor (touch order) & running force & speed-norm. \\\midrule
no sharing & 0.771 & 0.808 \\
hierarchical (random) & 0.673 & 0.765 \\
hierarchical (greedy) & 0.670 & 0.747 \\
outlier-aware (random) & 0.675 & 0.768 \\
outlier-aware (greedy) & 0.695 & 0.753 \\
\bottomrule\end{tabular}\end{table}

\begin{table}[t]\centering\small\setlength{\tabcolsep}{3pt}
\caption{Query ``will $F$\,N keep part $j$ moving?'' before touching $j$ ($\delta = 0.1$). False-safe: ``yes'' although more force is needed.}\label{tab:q4}
\resizebox{\linewidth}{!}{\begin{tabular}{lccccccc}\toprule
& & \multicolumn{3}{c}{per-object prior} & \multicolumn{3}{c}{\textbf{PhysVLMM} map belief} \\
grouping & $F$ & Brier & answered & false-safe & Brier & answered & false-safe \\\midrule
geometric & 3\,N & 0.188 & 0.00 & -- & 0.135 & 0.19 & -- \\
geometric & 5\,N & 0.247 & 0.00 & -- & 0.192 & 0.01 & 0.00 \\
geometric & 8\,N & 0.117 & 0.09 & 0.50 & 0.108 & 0.47 & 0.09 \\
semantic & 3\,N & 0.204 & 0.00 & -- & 0.131 & 0.32 & 0.00 \\
semantic & 5\,N & 0.218 & 0.00 & -- & 0.137 & 0.24 & 0.00 \\
semantic & 8\,N & 0.061 & 0.19 & 0.16 & 0.061 & 0.72 & 0.04 \\
\bottomrule\end{tabular}}\end{table}

\begin{table}[t]\centering\small\setlength{\tabcolsep}{3pt}
\caption{Grouping mistakes: each part is moved to a wrong look-alike group with probability $\rho$ (5 seeds). NLL difference outlier-aware $-$ hierarchical pooling with 95\,\% CI over homes (negative = outlier-aware better).}\label{tab:misgroup}
\begin{tabular}{ccc}\toprule
$\rho$ & running force & speed-normalised \\\midrule
0.0 & $-$0.04 [$-$0.16, +0.08] & +0.05 [$-$0.00, +0.14] \\
0.1 & $-$0.01 [$-$0.13, +0.12] & +0.05 [+0.00, +0.11] \\
0.2 & $-$0.05 [$-$0.15, +0.06] & +0.03 [+0.01, +0.06] \\
0.3 & $-$0.01 [$-$0.12, +0.11] & +0.02 [+0.00, +0.05] \\
0.5 & +0.01 [$-$0.09, +0.12] & +0.04 [+0.01, +0.09] \\
\bottomrule\end{tabular}\end{table}

\section{Per-group forces}\label{sec:s-groups}
\Cref{tab:s-groups} lists the measured running force of every part in look-alike groups of at least three parts.
\begin{table}[t]\centering\scriptsize\setlength{\tabcolsep}{3pt}
\caption{Running force (N) of every part in look-alike groups with at least 3 parts; homes are numbered, not named.}\label{tab:s-groups}
\resizebox{\linewidth}{!}{\begin{tabular}{llrl}\toprule
home & type & parts & running force per part (N) \\\midrule
H3 & revolute & 8 & 3.8, 1.7, 2.9, 2.3, 2.2, 3.0, 3.7, 3.2 \\
H4 & revolute & 3 & 1.5, 2.7, 1.5 \\
H4 & prismatic & 8 & 6.8, 10.1, 12.3, 4.1, 2.7, 3.7, 3.5, 3.0 \\
H5 & prismatic & 3 & 7.3, 5.5, 6.9 \\
H6 & prismatic & 5 & 6.5, 6.1, 1.9, 3.7, 2.2 \\
H7 & revolute & 6 & 3.3, 3.6, 3.5, 3.6, 4.8, 3.3 \\
H8 & revolute & 4 & 2.3, 3.1, 3.0, 4.3 \\
H9 & prismatic & 6 & 11.1, 6.5, 4.6, 6.8, 4.2, 5.3 \\
H9 & revolute & 8 & 9.5, 11.6, 3.8, 5.5, 3.8, 5.2, 3.7, 4.7 \\
H10 & prismatic & 4 & 6.1, 4.3, 3.6, 4.9 \\
H10 & prismatic & 3 & 3.4, 5.9, 4.6 \\
H10 & prismatic & 8 & 4.5, 3.4, 3.7, 6.9, 4.6, 3.2, 2.8, 3.6 \\
H11 & revolute & 10 & 9.6, 2.9, 5.2, 14.5, 4.4, 4.1, 3.1, 4.5, 2.8, 4.8 \\
H11 & prismatic & 3 & 3.1, 3.6, 3.4 \\
H12 & revolute & 3 & 2.3, 5.0, 2.4 \\
H12 & prismatic & 4 & 1.9, 1.9, 2.6, 1.9 \\
H13 & revolute & 4 & 2.5, 2.3, 2.3, 2.7 \\
H13 & prismatic & 3 & 1.8, 2.4, 2.5 \\
H14 & prismatic & 4 & 4.0, 3.4, 3.3, 17.3 \\
H15 & revolute & 8 & 5.6, 11.4, 14.7, 9.1, 6.3, 11.4, 6.4, 7.2 \\
H16 & revolute & 4 & 7.6, 8.0, 7.0, 6.1 \\
H17 & revolute & 5 & 2.6, 2.5, 14.9, 5.4, 3.8 \\
H17 & prismatic & 10 & 6.9, 5.1, 5.6, 6.8, 4.3, 3.3, 5.9, 10.5, 2.9, 3.4 \\
\bottomrule\end{tabular}}\end{table}

\section{Simulation track}\label{sec:s-sim}
Hoi! contains human pulls that mostly reach the stop, no collisions and only doors and drawers. A simulation track adds fixed robot actions, collision ground truth and a third part type, on one shared set of scenes for every method.

\paragraph{Scenes and physics.} RoboCasa kitchens~\cite{nasiriany2024robocasa} in ManiSkill~3~\cite{tao2025maniskill3} (SAPIEN/PhysX, CPU), \simKitchens{} kitchens (10 layouts $\times$ 12 styles), \simParts{} doors, drawers and microwave doors. Every part receives physics drawn once from a hierarchical model fitted by maximum likelihood on the annotated look-alike groups of Hoi! (kitchen effect, group effort $\mu_g$ around the category mean, part effort $x_i \sim \mathcal N(\mu_g, \omega^2)$ within a group of identical cabinets; the fitted kitchen variance is zero). This fit only builds the simulated world; no method reads it. The physics is realised as PhysX joint friction, joint damping and, for a share of doors, a closing spring. The Coulomb/viscous split is drawn per group, since Hoi! does not separate the two. PhysX friction depends on the joint load, unlike the constant friction of the model, so the simulator does not share the model's dynamics. Each part's friction coefficient is calibrated so that the measured effort matches the drawn one (median deviation 0.4\,\%); the measured effort is the ground truth.

\paragraph{Pulls and truth.} Per part, from the closed state: one training pull (force ramp to 20\,N) and four held-out pulls at a fixed 3, 5, 8 and 12\,N for 1\,s. The held-out pulls stop at a median of \simOpen\,\% of the stroke. Poses and forces carry the measurement noise measured on Hoi! (position \simNoisePos\,mm and rotation \simNoiseRot$^\circ$ RMS, correlated over \simNoiseTc\,s; force \simNoiseF\,N per axis); the truth is the simulator's joint position. A part counts as blocked when it stalls before 90\,\% of its range while touching another fixture (\simBlocked{} parts). \simJammed{} parts that stick in their own cabinet (asset overlaps) are excluded for every method, and the blocked parts and \simInvalid{} parts whose effort did not settle are excluded from the force and motion tasks (\simValid{} parts remain).

\paragraph{Perception and protocol.} One RGB-D view per part, 1.5\,m in front of it and 0.5\,m above it (closer if another fixture hides it), with the part mask from the simulator's segmentation (as the annotated part on Hoi!). Every method reads the same exported files, and every task is scored on the parts for which all compared methods return an output. Folds hold out one layout at a time (styles of a layout share fixtures), and confidence intervals come from bootstrapping layouts. FlowBot3D uses the orientation chosen on the Hoi! development homes, without re-tuning.

\paragraph{Results.} \Cref{tab:published,tab:sim-ablation} give the results. Notes:
\begin{itemize}
\item \emph{Joint type and axis} (\simTOneN{} parts). The joint is estimated from a controlled 20\,N training pull with the simulator's part mask, rather than from a person's pull, which makes it easier than on Hoi!.
\item \emph{Response.} RigPI returned no prediction for \simRigFail{} of the \simRigAttempt{} parts, and for \simNoFit{} more the 20\,N training pull reaches the stop so fast that too few samples remain for the dynamics fit; every method is scored on the \simRigN{} pulls of the remaining \simRigParts{} parts (\simRigKitchens{} kitchens). B1 fails because the fixed pulls stop short.
\item \emph{Force of an untouched part} (\simForceN{} parts). How alike look-alikes are is set by the Hoi! fit, so the size of this gain follows from that fit; the result to read here is the effect of grouping. In these kitchens, cabinets of different sizes sit side by side and face the same way, so the look-alike rule, applied to the rendered part points and the part labels as on Hoi!, merges true look-alike groups (\simMixed\,\% of parts are in a merged group). With the true groups, the motion CRPS of PhysVLMM is \simOursOracle.
\item \emph{Collision.} \simCollBlocked{} of the \simBlocked{} blocked parts have a joint estimate and are checked. With the part approximated by its bounding box, false alarms on free parts are \simFABox\,\%.
\end{itemize}

\end{document}